\documentclass{basic}
\begin{document}
	\title{\bf Change Point Detection with Conceptors}
	\author{Noah D. Gade\footnote{PhD Candidate, Department of Statistics, University of Virginia; Email: ndg5e@virginia.edu}\hspace{.2cm}\\
		and\\
		Jordan Rodu\footnote{Assistant Professor, Department of Statistics, University of Virginia; Email: jsr6q@virginia.edu}}
	\date{\today}
	\maketitle
	\bigskip
	\begin{abstract}
		Offline change point detection retrospectively locates change points in a time series. Many nonparametric methods that target i.i.d. mean and variance changes fail in the presence of nonlinear temporal dependence, and model based methods require a known, rigid structure. For the \emph{at most one change point} problem, we propose use of a conceptor matrix to learn the characteristic dynamics of a baseline training window with arbitrary dependence structure. The associated echo state network acts as a featurizer of the data, and change points are identified from the nature of the interactions between the features and their relationship to the baseline state. This model agnostic method can suggest potential locations of interest that warrant further study. We prove that, under mild assumptions, the method provides a consistent estimate of the true change point, and quantile estimates are produced via a moving block bootstrap of the original data. The method is evaluated with clustering metrics and Type 1 error control on simulated data, and applied to publicly available neural data from rats experiencing bouts of non-REM sleep prior to exploration of a radial maze. With sufficient spacing, the framework provides a simple extension to the sparse, multiple change point problem.
	\end{abstract}
	
	\noindent%
	\vfill
	\newpage
	\spacingset{1.3}
	
\section{Introduction}
The offline change point identification problem is widely discussed in the literature, and used in a broad range of domains like signal processing, human activity monitoring, and finance \citep{truong20}. For time series data $\mathbf{y}_t\in\mathbb{R}^d$ with $t = 1,\ldots, T$, the goal is to retroactively identify points where the distribution changes. Despite the vast amount of work in this field, most of it is focused on identifying mean and variance changes, and the challenge of change point detection in the presence of nonlinear dependence is unresolved. Model based methods assume a known, rigid structure, and specification of a functional form to fit nonlinearities present in data can be difficult. Many nonparametric methods are only applicable to independent and identically distributed data, and are applied to cases where changes would be easily identified from visual inspection of a time series plot. Unless they target the aforementioned mean and variance scenarios, most nonparametric methods also have an uninterpretable or shrouded definition of the notion of ``change.'' The contribution of this paper is a model agnostic method for detecting change in multivariate and arbitrarily dependent nonlinear time series data.
\par
The conceptor change point (CCP) methodology builds on prevalent ideas in statistics and representation learning literature. High-dimensional featurizations from echo state networks (ESNs), and a conceptor matrix from a specified and interpretable ``baseline'' state, allows for the flexibility of change detection in processes with elaborate dependence structures. This methodology should be used as a tool to suggest potential locations of interest in a dataset where traditional methods and inspection fail.
\par
For this work, the framework is restricted to the at most one change point (AMOC) problem. The distribution functions for each vector in time $\mathbf{y}_1, \ldots, \mathbf{y}_{\tau} \sim \mathcal{F}_{1}$ and $\mathbf{y}_{\tau+1}, \ldots, \mathbf{y}_T\sim \mathcal{F}_{T}$ are compared with the hypotheses
\begin{align}
	H_0 &: \mathcal{F}_{1} = \mathcal{F}_{T}\nonumber\\
	H_A &: \mathcal{F}_{1} \neq \mathcal{F}_{T} \label{ch1:hypothesis},
\end{align}
where $\mathcal{F}_{1}$ and $\mathcal{F}_{T}$ are unknown. Rejection of the null leads to the conclusion that a change took place immediately after time point $\tau$. With sufficient spacing, the framework provides simple extensions to the sparse, multiple change point problem and the online, sequential change point problem.

\subsection{Change Point Detection}
Change point detection methods can broadly be classified as sequential, agglomerative, or divisive. Sequential change point detection, like that of \citet{lai95}, best lends itself to the online problem where changes are identified in sequence while the data is observed. Online change point detection is not the main focus of this manuscript. Agglomerative change point methods begin by labelling each data point as a unique cluster, and proceed by strategically grouping adjacent clusters with an algorithmic criterion; \citet{fryzlewicz18} and \citet{matteson14} provide two examples of these criteria.
\par 
Divisive algorithms cluster a series of points and algorithmically search for breaks that best divide into chronological classifications. Many of these algorithms follow the binary segmentation approach, pioneered by \citet{vostrikova81}, and identify subsequent change points from each of the divided pieces. Penalized techniques like \citet{ombao05}, \citet{lavielle06}, and \citet{killick12} are also prevalent and recent computational work by \citet{haynes17} and \citet{tickle20} build on these methods and aim to optimize the search process for efficient change point detection. CUSUM type statistics provide the most common foundation to locate change points in ordered sequences \citep{picard85, gombay95, gombay99, cho12, holmes21, kojadinovic21}. Many modifications are in the form of the Kolmogorov-Smirnov and Cramer-von Mises criteria, including the self-normalization method of \citet{shao10}. Most initial work in this field, like wild binary segmentation of \citet{fryzlewicz14}, applies only to univariate sequences, and several extensions of these methods to multivariate and high-dimensional applications build on the respective univariate versions \citep{cho15}. \citet{matteson14} propose a clustering algorithm based on a hierarchical divergence measure for the multivariate, multiple change point problem. This method imposes the strict assumption of independent and identically distributed data and is only asymptotically justified in the AMOC problem \citep{arlot19}. Projection of the data into a high-dimensional space by \citet{wang18}, the kernel trick for change point estimation by \citet{arlot19}, and Bayesian estimation like \citet{cappello23} also import the assumption of independent data. \citet{dehling15} and \citet{dehling22} investigate relaxing the i.i.d. requirement by examining the effect of short-term dependence on large sample behavior, and \citet{gerstenberger18} explores only the mean change problem for short-term dependence. Certain types of model based detection like \citet{kirch15} also relax the i.i.d. requirement, but rely on strong parametric assumptions and impose a rigid structure onto the process. Characterizing nonlinear temporal dependence in change point problems is a challenging, relevant problem with relatively little progress.

\subsection{Conceptors}
An ESN expands an input sequence to a nonlinear, high-dimensional feature space that captures distinct characteristics of the data via a fixed, sparsely connected \emph{reservoir} of randomly intialized weights. The universal approximation theorems of \citet{cybenko89} and \citet{hornik91} have been extended to recurrent neural networks (RNNs) and ESNs with a sufficient reservoir size, and features that are not linearly separable in the original input space can become separable in the reservoir \citep{lukosevicius12, grig18, hart20, gonon21}. The reservoir is composed of randomly populated matrices, $\mathbf{W^h}$, $\mathbf{W^i}$, and $\mathbf{b}$ in Equation \ref{eq:CRNNhidden}, and the 
output sequence is obtained from a simple linear mapping ($\mathbf{W^o}$) of the feature space, usually trained  with Ridge regression \citep{jaeger02, jaeger07}. 
\par
\citet{jaeger14} introduced the conceptor as a regularized identity mapping of reservoir states in a propagating ESN. A matrix $\mathbf{C}$ is added to the update equation that recognize, control, regenerate, and predict state patterns in a dynamic reservoir \citep{jaeger14, jaeger17}. The matrix is placed at the front of Equation \ref{eq:CRNNhidden2} to filter the network states in a manner associated with a specific pattern.
\begin{align}
	\text{reservoir states: }\mathbf{h}_t &= g\left(\mathbf{W^h}\tilde{\mathbf{h}}_{t-1} + \mathbf{W^i}\mathbf{y}_t + \mathbf{b}\right)\label{eq:CRNNhidden}\\
	\text{filtered states: }\tilde{\mathbf{h}}_t &= \mathbf{C} \mathbf{h}_t\label{eq:CRNNhidden2}\\
	\text{reservoir output: }\hat{\mathbf{y}}_t &= \mathbf{W^o} \tilde{\mathbf{h}}_t\label{eq:CRNNoutput}
\end{align}
The conceptor matrix $\mathbf{C}$, defined in Equation \ref{eq:ccalc}, is computed from column space of the reservoir states over some training window of data (of integer length $T_{\text{train}}$) that exhibits a pattern of interest. For dynamics associated with the pattern from the training window, the conceptor matrix should leave the state unchanged and act like the identity, or $\tilde{\mathbf{h}}_t \approx \mathbf{h}_t$ \citep{jaeger17}. For states propagating in a fashion atypical for the pattern from the training window, the conceptor matrix will act like the null matrix and suppress the reservoir state \citep{jaeger17}. 
\par
Over the training window, $\mathbf{H}_{\text{train}}\in\mathbb{R}^{T_\text{train} \times N}$ compiles the reservoir states as individual rows (where prior to computation $\mathbf{C} = \mathbf{I}$ in the training window). The conceptor is a positive semidefinite matrix ($\mathbf{0} \preceq \mathbf{C} \preceq \mathbf{I}$) that forms an ellipsoid within the unit sphere. The principal axes of the ellipsoid are the eigenvectors of the reservoir state second moment matrix $\mathbf{R} = T_{\text{train}}^{-1}\mathbf{\mathbf{H}_\text{train}}'\mathbf{\mathbf{H}_\text{train}}$ that have been scaled by their eigenvalues. 
\begin{align}
	\mathbf{C} \equiv\mathbf{R}\left(\mathbf{R} + \alpha^{-2}\mathbf{I}\right)^{-1}\label{eq:ccalc}
\end{align}
A regularization parameter $\alpha\in\left(0,\infty\right)$, known as the aperture, influences the degree of dampening in the reservoir activity \citep{jaeger14, jaeger17}. Large apertures introduce small penalties, and the conceptor will tend to the identity with minimal activity dampening; smaller $\alpha$ shrinks the conceptor to the zero matrix.
\par 
In many applications, conceptors are appealing for prediction and pattern regeneration tasks because it is possible to remove the input data from the update equation. Equation \ref{eq:CRNNhidden} is modified to $\mathbf{h}_t = g\left(\mathbf{W} \tilde{\mathbf{h}}_{t-1} + \mathbf{b}\right)$, where $\mathbf{W}$ is trained from the input and reservoir states, and the conceptor governs the dynamics of the system autonomously \citep{jaeger14, jaeger17}. 
\par 
In this work, conceptors are used to encapsulate the dynamics of a training window of data. The ESN featurization serves as an advantageous functional transformation to a high-dimensional domain where nonlinear relationships are ``linearized'' without imposing a rigid structure. A specified training window serves as an interpretable baseline state, the conceptor records information about the baseline dynamics, and the relationship between the reservoir states and filtered states is used to highlight differences in the mechanism of evolution. 
\par 
Nonlinear temporal dependence is difficult to characterize, especially in multivariate time series. Linear tools often oversimplify the dependence structure, and choosing a ``close enough'' functional form can be a challenging task. The CCP methodology allows for flexible change point detection in processes with elaborate dependence structures. 

\section{Methodology}\label{se:CCPmethods}
The hypothesis in Equation \ref{ch1:hypothesis} is tested under the condition where at most one change point $\tau$ is present in the data. The selected training window should be sufficiently long to capture the original dynamics of the time series and include representative values from the data prior to a suspected change. If the data exhibits a periodic or almost-periodic type structure, the training length should include at least one full cycle. It is assumed that any change takes place after this baseline window, and the initial distribution $\mathcal{F}_{1}$ produces a time series that is at least wide-sense cyclostationary. The second assumption ensures the training data covers a relevant range of the data and changes are not falsely identified from unrecognized network dynamics. The training window need not be at the beginning of the time series; this can be generalized to take any section as a baseline state, where the method looks forward and backward for a potential change.
\par
The proposed method involves three main steps. First, several ESNs are generated and the specified training window is used to select network parameters that satisfy an error tolerance $\varepsilon_\text{train}$, defined in terms of the normalized root mean square error (NRMSE) of the reservoir output, $\text{NRMSE} = \left[\left(\mathbf{y}_t -\hat{\mathbf{y}}_t\right)^2/\left(\frac{1}{2}\text{Var}\left(\mathbf{y}_t\right) + \frac{1}{2}\text{Var}\left(\hat{\mathbf{y}}_t\right)\right)\right]^{1/2}$. For each ESN, a conceptor matrix is computed that encodes information about the dynamics of the training window of a time series. Second, the relative differences between the reservoir and filtered states are examined with a bisection technique that proposes a change point. Last, a moving block bootstrap is used to estimate the strength of evidence for a proposed change. 

\subsection{ESN Featurization}
Along with an integer length for the baseline $T_\text{train}$, an integer length $T_\text{wash}$ is specified and used to washout the initial conditions of a generated ESN reservoir, where generally $T_\text{wash} < T_\text{train} \ll T$. A training error $0<\varepsilon_\text{train}\ll1$ influences hyperparameter settings; this error tolerance is the maximum allowable NRMSE between the conceptor governed ESN output and the original data.  Define $T_0=T_\text{wash} + T_\text{train}$, and the assumption of no distributional change applies to the time points where $t\leq T_0$, restricting the identification of any change point $\tau$ to the interval $[T_0 + 1, T-1]$.
\par
A reservoir size $N$ is selected where $N\gg d$. A series of $r = 1,\ldots,\mathscr{R}$ ESNs are initialized by generating the matrices from Equation \ref{eq:CRNNhidden}. The input and bias matrices, $\mathbf{W^i}_r\in\mathbb{R}^{N\times d}$ and $\mathbf{b}_r\in\mathbb{R}^{N}$, are dense with independent random Gaussian realizations, where the variance is determined from a parameter grid such that the ESN output best fits the data as measured by NRMSE. Each $\mathbf{W^h}_r\in\mathbb{R}^{N\times N}$ is a sparse matrix with independent random Gaussian nonzero entries, and is scaled to a constant spectral radius of 0.8 to ensure the propagating reservoir states remain stable and wash out (or ``forget'') the information from any initial conditions. The chosen spectral radius is not at the limit of stability as a balance between introducing adequate temporal dependence and ensuring the reservoir washout does not waste valuable data \citep{lukosevicius12, yildiz12}.
\par
The reservoir size and aperture of the network are concurrently determined as the first values large enough to produce a NRMSE below $\varepsilon_\text{train}$. Selection of $\varepsilon_\text{train}$ should take place prior to analysis and not with iterations of the method on multiple parameter values. Smaller values may be more sensitive to learning variations in the noise component of the data, while larger values will extract general behavior. In Section \ref{se:CCPsims}, the effect of varying $\varepsilon_\text{train}$ is demonstrated, but the parameter should take a default of about 4 to 8\% NRMSE unless compelling prior knowledge about the type of change sought suggests otherwise. Parameter selection procedures and the full ESN featurization algorithm are given in the supplement.
\par
For each ESN featurization, the associated conceptor matrix is computed from the series of training time points after network washout, $t\in[T_\text{wash} + 1,T_0]$. The network is propagated forward in time via Equations \ref{eq:CRNNhidden} and \ref{eq:CRNNhidden2} with $\mathbf{C} = \mathbf{I}$, and the reservoir states $\tilde{\mathbf{h}}_t$ ($\mathbf{h}_t$) collected. Conceptor matrices $\mathbf{C}_r$ are obtained via Equation \ref{eq:ccalc}; then, each ESN is run with $\mathbf{C}_r$ in place, like Equation \ref{eq:CRNNhidden2}, for all time points $t\geq T_\text{wash} + 1$.

\subsection{Change Point Proposal}\label{ss:CPP}
The angles between the filtered states $\tilde{\mathbf{h}}_{r,t}$ and the reservoir states $\mathbf{h}_{r,t}$ are examined at each time $t$ after the training period. The cosine similarities of these angles $s_{r,t}$, for each featurization $r$, form univariate sequences that quantify the proximity of the reservoir states to the space spanned by the corresponding conceptor matrix.
\begin{align}
	s_{r,t} &= \frac{\tilde{\mathbf{h}}_{r,t}' \mathbf{h}_{r,t}}{\left|\left|\tilde{\mathbf{h}}_{r,t} \right|\right|\left|\left|\mathbf{h}_{r,t}\right|\right|} = \frac{\mathbf{h}_{r,t}' \mathbf{C}_r' \mathbf{h}_{r,t}}{\left|\left|\mathbf{C}_r\mathbf{h}_{r,t} \right|\right|\left|\left|\mathbf{h}_{r,t}\right|\right|}\label{eq:similarity}
\end{align}
Values in each sequence $s_{r,t}$ are contained in the interval $[0,1]$ because, by definition, each $\mathbf{C}_r$ is positive semidefinite. A similarity value can be interpreted as a measure of the strength of relationship between the reservoir state at time $t$ and those in the period of training data. Values of zero imply the ESN is generating states orthogonal to the conceptor space, and those equal to one imply the ESN is generating states exactly in the conceptor space. 
\par
Other distance measurements may be used to quantify proximity to the conceptor space. Cosine similarity is a bounded, interpretable quantity that emphasizes angles further from zero at an increasing rate. The exact values of these similarities will vary and their absolute measure is not important; only relative differences are needed for comparison throughout the time series. Figure \ref{CCPexample1} and others in the supplement illustrate the relative nature of the similarity measure.
\par
Because the networks are randomly generated, there is variation in the computed conceptor matrices that extends to the cosine similarities across the ESNs. To extract a general behavior and reduce the dimension of the information, the average cosine similarity at each time point $t$ is considered, $S_t = \mathscr{R}^{-1}\sum_{r=1}^\mathscr{R} s_{r, t}$. The aggregate cosine similarity sequence acts as an ensemble of weak learners from each generated ESN.
\par
A proposed change point is selected using a modified CUSUM statistic that resembles the two-sample Kolmogorov-Smirnov distributional test \citep{smirnov33}. From the sequence $S_t$, empirical CDFs  $\hat{\mathcal{F}}_{(T_0+1):t}(s)$ and $\hat{\mathcal{F}}_{(t+1):T}(s)$ are constructed by dividing the sample at each potential change point in the series.
\begin{align}
	\hat{\mathcal{F}}_{(T_0+1):t}(s) &= \frac{1}{t - T_0}\sum_{i=T_0 + 1}^{t} \mathbf{1}\left\{S_i \leq s\right\}\label{eq:ecdf1}\\
	\hat{\mathcal{F}}_{(t+1):T}(s) &= \frac{1}{T - t}\sum_{i=t + 1}^{T} \mathbf{1}\left\{S_i \leq s\right\}\label{eq:ecdf2}
\end{align}
A scaled statistic, like that used in \citet{gombay95} and \citet{padilla19}, is computed at each observation, and the point of the maximum is identified as the most likely change point:
\begin{align}
	K &= \max_{t} \frac{\left(t-T_0\right)\left(T - t\right)}{q\left(t\right)\left(T-T_0\right)^{2}} \sup_s \left|\hat{\mathcal{F}}_{(T_0+1):t}(s) - \hat{\mathcal{F}}_{(t+1):T}(s)\right| \label{eq:statistic}\\[4pt]
	\hat{\tau} &= \arg \max_{t} \frac{\left(t-T_0\right)\left(T - t\right)}{q\left(t\right)\left(T-T_0\right)^{2}} \sup_s \left|\hat{\mathcal{F}}_{(T_0+1):t}(s) - \hat{\mathcal{F}}_{(t+1):T}(s)\right|\label{eq:changepoint}.
\end{align}
The coefficient term in Equation \ref{eq:statistic} ensures the statistic converges in distribution under stationarity as $T\rightarrow\infty$ \citep{csorgHo97b, gombay99, holmes13, kojadinovic21}. The $q(t)$ scaling function, with form shown in Equation \ref{eq:scaling}, increases the sensitivity of the method near the edges of the sequence \citep{csorgHo94a, csorgHo94, csorgHo97b, csorgHo97}. Further properties of the statistic under the null are discussed in Section \ref{se:CCPtheory}.
\begin{align}
	q(t) &= \max \left\{ \left(\frac{t-T_0}{T-T_0}\right)^\nu\left(1-\frac{t-T_0}{T-T_0}\right)^\nu, \;\kappa\right\} \label{eq:scaling}
\end{align}
The statistic given in Equation \ref{eq:statistic} is similar to that discussed in \citet{kojadinovic21}, where if $\nu = 1/2$ and $\kappa$ is a small constant near zero, the mean and variance of the series of statistics remain approximately constant in the limit. Values $\nu=1/2$ and $\kappa=0.01$ are chosen for reasons explained in Section \ref{se:CCPtheory}, and the full algorithmic process for proposing a change point is detailed in the supplement.

\subsection{Moving Block Bootstrap}
Estimating the null distribution of the statistic $K$ from Section \ref{ss:CPP} may be of more importance than the identification of a potential change point. By nature, many change point detection problems are non-verifiable in real world applications. Thus, reliable change point algorithms should be robust in their ability to detect both change and the lack of change in a dataset.
\par
The bootstrap method of \citet{efron79} provides a foundation for inference based on repeated sampling. Several variations of the bootstrap for change point methodology are reviewed in \citet{huvskova04}. The moving block bootstrap (MBB), developed by \citet{kunsch89} and \citet{liu92}, samples blocks of consecutive points to retain the dependence structure within each block. Extending the bootstrap to applications with dependent data structures, the MBB is able to asymptotically reproduce the underlying dependence structure \citep{lahiri03}. The block bootstrap is suggested by \citet{dehling02} for statistics of empirical processes, and \citet{synowiecki07} describes its application to non-stationary data with periodic or almost periodic behavior.
\par
Potentially overlapping blocks of length $L$ are selected and combined to form a bootstrapped time series of the original length $T$. When data are treated as i.i.d., as in \citet{matteson14}, the block length parameter reduces to a permutation of the time series with $L=1$. With increasing block length, the estimated null distribution will exhibit less variation, and null values will tend closer to the statistic $K$. Choice of block length is discussed widely in literature, and the cross-validation like technique of \citet{hall95} allows for data driven selection \citep{buhlmann99, lahiri03, politis04, lahiri07, patton09}. A large pilot block length is used to start the \citet{hall95} algorithm to ensure a long range dependence structure is considered, and the process is not iterated as convergence is not guaranteed \citep{lahiri03}. Adjusting the block length for the power and Type 1 error control trade-off should be considered if the researcher possesses some knowledge about the dependence structure of the data.
\par
With the assumption of no change in the washout and training windows, the conceptor matrices do not need to be recomputed: in each bootstrapped series the points $t\leq T_0$ are identical to the original data. All remaining points $t\geq T_0 + 1$ are equally likely to be chosen as the beginning of a bootstrap interval of length $L$. To ensure equal inclusion probability, the series is wrapped such that a block near the end (a time point within $L$ of $T$) will cycle back to the initial time considered, $T_0 + 1$.
\par
For each generated bootstrap time series $b = 1,\ldots,B$, the corresponding maximum statistic is computed as in Section \ref{ss:CPP}, simulating an approximate null distribution. A distribution quantile is estimated from the fraction of the $B$ bootstrapped statistics that exceed the statistic from Equation \ref{eq:statistic}, $p = B^{-1} \sum_{b=1}^B \mathbf{1}\left\{K_b > K \right\}$. The quantile estimate provides a notion of the strength of the evidence for a change point. The validity of applying the MBB to the AMOC problem is investigated through simulation and evaluation of Type 1 error control in Section \ref{se:CCPsims} for a variety of data generating processes. For a dataset with no change point, the proportion of false rejections in $\mathscr{S}$ simulations is expected to be approximately $q \mathscr{S}$, where $q$ is a predefined threshold of Type 1 error. The full algorithm for estimating a null distribution from the MBB is in the supplement.

\section{Theory}\label{se:CCPtheory}
The hypothesis in Equation \ref{ch1:hypothesis} is tested using the featurization and conceptor matrix as outlined in Section \ref{se:CCPmethods}. Under the null hypothesis, the cosine similarity values $S_t$ are expected to retain a consistent relative structure and fall close to one (\textit{i.e.}, remain close to the space of the conceptor matrix). Under the alternative, changes in the relationship between the cosine similarity sequence and the space spanned by the conceptor matrix are observed. These changes, initiated by a shift in the data, are not strictly away from the conceptor space; a reduction in variation may lead to reservoir states that lie closer to the conceptor space.
\par
For clarity, the time index is redefined to $T_0=0$ and $T$ as the number of data points after washout and training. The assumption of any change after washout and training restricts the domain to $t>0$. Suppose $S_t \sim \mathcal{F}_t(s)$ for all $s\in[0,1]$, where each $\mathcal{F}_t$ is a defined distribution function. The hypothesis is reformulated in terms of these distribution functions.
\begin{align}
	H_0&: \mathcal{F}_1(s) = \cdots =\mathcal{F}_T(s) \text{ for all $s\in[0,1]$}\nonumber\\
	H_A&: \text{$\exists$ $\tau\in\mathbb{Z}$, $1\leq \tau < T$, such that $\forall$ $s\in[0,1]$, }\mathcal{F}_1(s) = \cdots =\mathcal{F}_\tau(s)\text{ and }\nonumber\\
	&\hskip0.15in \mathcal{F}_{\tau+1}(s) = \cdots =\mathcal{F}_T(s), \text{ and } \mathcal{F}_{1}(s_0) \neq \mathcal{F}_{T}(s_0) \hskip0.1in\text{for some $s_0\in[0,1]$} \label{eq:hypothesis2}
\end{align}
\par
Define the corresponding empirical distribution functions at each point in the time series $t$, $\hat{\mathcal{F}}_{1:t}(s) = t^{-1}\sum_{i=1}^{t} \mathbf{1}\left\{ S_i \leq s\right\}$ and $\hat{\mathcal{F}}_{(t+1):T}(s) = \left(T-t\right)^{-1}\sum_{i=t + 1}^{T} \mathbf{1}\left\{ S_i \leq s\right\}$. By the Glivenko-Cantelli theorem, these empirical CDFs are consistent estimators of the true distribution functions and they uniformly converge in the limit under stationarity and ergodicity \citep{tucker59, yu93, dehling02}. Summarizing the data with the univariate sequence generated by the conceptor space $S_t$, rather than using the original multivariate time series, may also make investigation and verification of theoretical assumptions more accessible. 

\subsection{Limiting Distribution Under the Null Hypothesis}\label{ss:LimitDist}
Asymptotic behavior of empirical processes for i.i.d. sequences stems from \citet{kiefer72}, that proved almost sure convergence to a Gaussian process. In this work, the $\mathcal{S}$-mixing definition of \citet{berkes09} is used to obtain a similar result for the statistic in Equation \ref{eq:statistic}.
\par
Assume the sequence $S_t$ is stationary under the null hypothesis, satisfied by construction in Equation \ref{eq:hypothesis2}, and can be represented as a shift process of i.i.d. random variables $\varepsilon_t$, $S_t = f(\varepsilon_t, \varepsilon_{t-1},\ldots)$. Most stationary processes in practice admit a representation as a shift process, and this is causal due to the forward dynamics of the time series \citep{berkes09}. These assumptions apply only to the sequence indicating the relationship of a given state to the baseline space spanned by the conceptor matrix; they are not imposed on the original data. In a time window without a change point present, assume $S_t$ will arise from some common distribution. A process as $\mathcal{S}$-mixing if the two conditions in Definition \ref{def:smix} are satisfied. 
\begin{definition}
	A random process $S_t$ is $\mathcal{S}$-mixing if:
	\begin{enumerate}
		\item[(1)] For any $t\in\mathbb{Z}$ and $m\in\mathbb{N}$, one can find a random variable $S_{tm}$ such that \\$P(|S_t-S_{tm}|\geq \gamma_m) \leq \delta_m$ for some numerical sequences $\gamma_m\rightarrow 0$, $\delta_m\rightarrow 0$.
		\item[(2)] For any disjoint intervals $\mathcal{I}_1,\ldots,\mathcal{I}_r$ of integers and any positive integers \\$m_1,\ldots,m_r$, the vectors $\{S_{jm_1},j\in\mathcal{I}_1\},\ldots,\{S_{jm_r},j\in\mathcal{I}_r\}$ are independent \\provided the separation between $\mathcal{I}_1$ and $\mathcal{I}_r$ is greater than $m_1+m_r$.
	\end{enumerate}
	\label{def:smix}
\end{definition}
With mild assumptions on the function $f$, \citet{berkes09} easily show the shift process representation for a general class of nonlinear processes. Construction of the approximating sequence $S_{tm}$ is discussed via substitution, truncation, coupling, and smoothing techniques \citep{berkes09}. $\mathcal{S}$-mixing is not directly comparable to classical mixing conditions, like $\alpha$-, $\beta$-, or $\rho$-mixing. The classical mixing definitions lead to clean and precise theoretical results, but verifying the required conditions can be challenging and their scope of application is limited \citep{berkes09}. $\mathcal{S}$-mixing relaxes these requirements to the existence of an approximating sequence that satisfies the above properties. Within the targeted class of shift processes, verification of assumptions is almost immediate, and the resulting strong approximation is used to derive the limiting distribution \citep{berkes09}.
\par
Define the function $q:[0,1]\to (0,1)$ by $q(\delta) = \max\{\delta^{1/2}(1-\delta)^{1/2}, \kappa\}$ and some small $\kappa>0$ resembling Equation \ref{eq:scaling} with $\nu = 1/2$.
\begin{theorem}
	Let $S_t$ be a stationary sequence such that $\mathcal{F}(s) = P(S_1 \leq s)$ is Lipschitz continuous of order $C > 0$. Assume $S_t$ is $\mathcal{S}$-mixing and that condition (1) of Definition \ref{def:smix} holds with $\gamma_m=m^{-AC}$, $\delta_m = m^{-A}$ for some $A>4$. Under the null hypothesis for every $\kappa\in\left(0,\frac{1}{2}\right)$,
	\begin{align}\label{mainthmresult}
		\sqrt{T}\,\max_{1 \leq t < T} \frac{1}{q\left(\frac{t}{T} \right)}\left[\frac{t(T-t)}{T^2} \right] \sup_{s\in[0,1]}\left|\hat{\mathcal{F}}_{1:t}(s) - \hat{\mathcal{F}}_{(t+1):T}(s)\right| \xrightarrow{D} \sup_{\delta \in [0,1]}  \sup_{s\in[0,1]} \left| \mathcal{K}(s, \delta)\right| / q(\delta)
	\end{align}
	as $T\to\infty$, where $\mathcal{K}(s, \delta)$ is a Gaussian process with
	\begin{align}\label{mainthmdetails}
		&\mathbb{E}\left[\mathcal{K}(s,\delta)\right] = 0,\nonumber\\
		&\mathbb{E}\left[\mathcal{K}(s,\delta)\;\mathcal{K}(s',\delta')\right] = (\delta \wedge \delta')\; \Gamma(s,s'),\nonumber\\[8pt]
		\text{and} \hskip0.2in&\Gamma(s,s') = \sum_{-\infty < t < \infty} \mathbb{E}\left[S_1(s) S_t(s')\right],
	\end{align}and the limiting random variable is almost surely finite.
	\label{mainthm}
\end{theorem}
The mathematical exposition and proof is an extension of the independent case found in \citet{csorgHo97b}, Theorem 2.6.1 and can be found in the supplement. Theorem \ref{mainthm} implies the statistic $K$ converges in probability to zero under the null hypothesis.
\par
Stationarity of the average cosine similarities depends on the training window of data and adherence to the AMOC problem. With a well specified, sufficiently long training window such that a relevant range of the time series is covered, the reservoir will emit dynamics close to the conceptor space and the assumption is likely satisfied. With multiple changes present in a dataset, the stationarity assumption may be violated. In practice, the data should be at least wide-sense cyclostationary, contain at most one change, and not exhibit some long run trend. 

\subsection{Consistent Change Point Estimation}\label{ss:ConsistentCP}
The behavior of the change point estimate $\hat{\tau}$ is examined under the general class of alternatives given in the hypothesis of Equation \ref{eq:hypothesis2}. Under construction of the alternative, the average cosine similarity sequence divides into two stationary ergodic pieces on either side of a true change point $\tau$ represented by $\mathcal{F}_{1}(s)$ and $\mathcal{F}_{T}(s)$. When satisfying modest conditions, $\hat{\tau}$ is a consistent estimator of $\tau$.
\begin{theorem}\label{consistent}
	Suppose the sequence $S_t$, $1,\ldots,T$ divides into two stationary ergodic pieces on either side of the change point $\tau$, and $\mathcal{F}_{1}(s_0) \neq \mathcal{F}_{T}(s_0)$ for some $s_0\in[0,1]$. Then for every $\kappa\in\left(0,\frac{1}{2}\right)$, the change point estimate
	\begin{align}
		\hat{\tau} &= \arg \max_{1\leq t < T} \frac{1}{q(\frac{t}{T})}\left[\frac{t(T-t)}{T^2}\right]\sup_{s\in[0,1]}\left|\hat{\mathcal{F}}_{1:t}(s) - \hat{\mathcal{F}}_{(t+1):T}(s)\right| 
	\end{align}
	converges in probability to the true value $\tau$ under the domain restriction
	\begin{align}
		\tau \in \left[\frac{T}{2}-\frac{T}{2}\sqrt{1-4\kappa^2}, \frac{T}{2}+\frac{T}{2}\sqrt{1-4\kappa^2}\right].
	\end{align}
\end{theorem}
The proof in the supplement follows Theorem 2.1 from \citet{newey94} for consistency of extremeum estimators. Restricting a possible change point to the interval shown in Theorem \ref{consistent} does not shrink the domain in practice if the chosen $\kappa < \sqrt{\frac{1}{4} - \frac{1}{4}\left(1-\frac{2}{T}\right)^2}$.

\section{Simulation Study}\label{se:CCPsims}
Performance of the CCP method is demonstrated through simulations restricted to the AMOC problem. The CCP method is compared to the e-divisive (EDiv) method of \citet{matteson14} and the kernel change point (KCP) algorithm of \citet{arlot19}, both via the \verb|ecp| R package by \citet{james13}, along with the sparsified binary segmentation (SBS) methods of \citet{cho15}, via the \verb|sbs| R package by the same authors. Type 1 sparsified binary segmentation searches for changes in the center of the data, and Type 2 seeks other forms of distributional change \citep{cho15}.

\subsection{Simulation Settings}
Simulated time series fall into the broad classes of VAR, periodic, Gaussian, and white noise processes. All simulated data $\mathbf{y}_t\in\mathbb{R}^2$, $t=1,\ldots,T$, has length $T=1000$ with a potential change located from $\tau=181$ to $\tau=999$. Table \ref{CCPsimsettings} summarizes the settings used for each method in the study. CCP requires specification of a training length of data with an associated training error tolerance.
\begin{table}[htb]
	\centering
		\begin{tabular}{cc}
			\hline
			Method & Settings\\
			\hline
			Conceptor Change Point (CCP) & $\varepsilon_\text{train} = 2, 4, 8, 16$\\
			E-Divisive (EDiv) &  $q = 0.05$\\
			Kernel Change Point (KCP) & $\Pi = 1$, $C = 2$\\
			Sparsified Binary Segmentation (SBS) & $q = 0.05$, $\text{Type}=1,2$\\
			\hline
		\end{tabular}
		\caption{Parameter settings for methods in simulation study. $\varepsilon_\text{train}$ is the error tolerance in \% NRMSE, $q$ the significance threshold, $\Pi$ the maximum number of change points, $C$ the KCP penalty scaling.}
	\label{CCPsimsettings}
\end{table}
The washout length $T_\text{wash}=60$ and the training length $T_\text{train}=120$ are fixed to ensure a constant window of estimation $\hat{\tau}\in[181,999]$ for all compared methods. The error tolerance $\varepsilon_\text{train}$ varies from 2 to 16 percent of NRMSE. The EDiv, KCP, and SBS methods are restricted to the AMOC framework. For EDiv and SBS the estimate is the initial segmentation chosen by the algorithm. KCP accepts an input parameter to restrict to the AMOC problem. EDiv, SBS1, and SBS2 require a significance threshold for change point identification that is set to $q=0.05$. The same value is used in the CCP method to set an upper threshold on the bootstrap null distribution. KCP requires specification of a penalty parameter for change point identification; \citet{arlot19} outline a procedure for selection of this parameter, and the suggested penalty scaling is used.
\par
When a change in the data is present, the adjusted Rand index (ARI) of \citet{hubert85} compares the assignment of time points to the correct class, and the empirical CDF of the difference between the identified point and the true change point is computed. Given in Equation \ref{graphcdf}, with $\delta$ the fraction of the time series away from the true change point and $\mathscr{S}$ the total number of simulations for a selected setting, the empirical CDF shape in some neighborhood $\{\delta: 0\leq\delta\leq T^*\ll T\}$ compares performance of the methods.
\begin{align}
	\hat{\mathcal{H}}_\tau(\delta) = \frac{1}{\mathscr{S}}\sum_{i=1}^\mathscr{S} \mathbf{1}\left\{\frac{1}{T - T_\text{wash} - T_\text{train}}\left|\hat{\tau}_i-\tau_i\right| \leq \delta\right\}\label{graphcdf}
\end{align}
Better performing methods will quickly increase to 1, and those that fail to identify an existing change point are evaluated as if it was placed at the end of the series, $\hat{\tau} = 1000$. When no change is present, $\hat{q}$ is defined as the observed Type 1 error and compared with the defined threshold $q$. 
\par
Tables \ref{simdetails1} and \ref{simdetails2} detail the datasets examined in the simulation study. For each case, a no change point scenario is included where the initial data generating process held constant for the full time series. Each setting is indicated by a unique ID and repeated $300$ times to create over $13000$ simulated datasets.
\begin{table}[htb]
	\centering
		\begin{tabular}{ccccc}
			\hline
			ID & Simulated Data & & ID & Simulated Data \\\hline
			(1a); (2a) & $\rho = 0.5 \rightarrow 0.5$ & & (3a) & $\omega = 1 \rightarrow 0.5$\\
			(1b); (2b) & $\rho = 0.5 \rightarrow 0.8$ & & (3b) & $\omega = 1 \rightarrow 0.8$\\
			(1c); (2c) & $\rho = 0.8 \rightarrow 0.5$ & & (3c) & $\omega = 1 \rightarrow 1.2$\\
			(1d); (2d) & $\rho = 0.8 \rightarrow 0.8$ & & (3d) & $\omega = 1 \rightarrow 1.5$\\
			(1e); (2e) & $\rho = 0.5 \rightarrow \text{NC}$ & & (3e) & $\omega = 1 \rightarrow \text{NC}$\\
			(1f); (2f) & $\rho = 0.8 \rightarrow \text{NC}$ & & &  \\\hline
		\end{tabular}
		\caption{$\text{VAR}(1)+\frac{1}{2}\mathcal{N}_2\left(\mathbf{0}_2,\mathbf{I}_2\right)$, $\text{VAR}(2)+\frac{1}{2}\mathcal{N}_2\left(\mathbf{0}_2,\mathbf{I}_2\right)$ spectral radius $\rho$ change simulations, and periodic process frequency $\omega$ change simulations $\sin\left(\omega t \left\{{\scriptstyle+\omega\frac{\pi}{2}}\right\}\right)\mathbf{1}_2+\frac{1}{2}\mathcal{N}_2\left(\mathbf{0}_2,\mathbf{I}_2\right)$. All data $\mathbf{y}_t\in\mathbb{R}^2$, $t=1,\ldots,T$, has length $T=1000$ and the change point varies randomly $\tau\in[181,999]$ or no change (NC). VAR(1) simulations are indicated by ID(1), VAR(2) by ID(2), and periodic by ID(3).}
	\label{simdetails1}
\end{table}
\begin{table}[htb]
	\centering
		\begin{tabular}{ccccc}\hline
			ID & Simulated Data & & ID & Simulated Data \\\hline
			(4a) & $\theta = 0.5\rightarrow 0; \lambda = 0.5$ & & (5a) & $\mu = 0 \rightarrow 0.5$\\
			(4b) & $\theta = 0.5\rightarrow 1; \lambda = 0.5$ & & (5b) & $\mu = 0 \rightarrow 0.8$\\
			(4c) & $\theta = 1\rightarrow 0; \lambda = 0.5$ & & (5c) & $\mu = 0 \rightarrow 1$\\
			(4d) & $\theta = 1\rightarrow 0.5; \lambda = 0.5$ & & (5d) & $\sigma = 1 \rightarrow 0.5$\\
			(4e) & $\theta = 0.5; \lambda = 0.5 \rightarrow 0.2$ & & (5e) & $\sigma = 1 \rightarrow 0.8$\\
			(4f) & $\theta = 0.5; \lambda = 0.5 \rightarrow 0.8$ & & (5f) &  $\sigma = 1 \rightarrow 1.2$\\
			(4g) & $\theta = 0.5; \lambda = 0.5 \rightarrow 1$ & & (5g) &  $\sigma = 1 \rightarrow 1.5$\\
			(4h) & $\theta = 0.5; \lambda = 0.5 \rightarrow \text{NC}$ & & (5h) &  $\rho = 0 \rightarrow 0.8$\\
			(4i) & $\theta = 1; \lambda = 0.5 \rightarrow \text{NC}$ & &(5i) & $\mu, \rho = 0; \sigma = 1\rightarrow \text{NC}$ \\
			\hline
		\end{tabular}
		\caption{Ornstein-Uhlenbeck mean reverting $\theta$ and volatility $\lambda$ change simulations $\mathcal{OU}_2\left(\gamma \mathbf{I}_2,\lambda^2\mathbf{I}_2\right)$ and white noise $\mathcal{N}_2\left(\mathbf{0}_2 + \mu\mathbf{1}_2,\sigma^2\mathbf{I}_2 + \rho\mathbf{J}_2\right)$ mean $\mu$, variance $\sigma$, and covariance $\rho$ change simulations, where $\mathbf{J}_2$ refers to the anti-diagonal matrix of ones. All data $\mathbf{y}_t\in\mathbb{R}^2$, $t=1,\ldots,T$, has length $T=1000$ and the change point varies randomly $\tau\in[181,999]$ or no change (NC). Ornstein-Uhlenbeck simulations are indicated by ID(4) and white noise by ID(5).}
	\label{simdetails2}
\end{table}
\par
For VAR$(\gamma)$ processes, the coefficient matrix is randomly generated to have a fixed spectral radius $\rho$ (within a tolerance of 0.02). Change points from autoregressive processes with similar $\rho$ may be more difficult to identify as they can exhibit similar dynamics. All VAR$(\gamma)$ processes contain a white noise error term defined in Table \ref{simdetails1}. For periodic processes, the second dimension is shifted by of $\pi/2$ as noted by the braced parenthesis in Table \ref{simdetails1}, and all contain a white noise error term. The Ornstein-Uhlenbeck processes are defined by the stochastic differential equation $d\mathbf{x}_t = \theta \mathbf{x}_t dt + \lambda d\mathcal{W}_t$, where $\mathcal{W}_t$ denotes a two-dimensional Wiener process. The two-dimensional Ornstein-Uhlenbeck process is denoted $\mathcal{OU}_2(\Theta,\Lambda)$, where $\Theta$ is the $2\times 2$ mean-reverting matrix and $\Lambda$ is the $2\times2$ volatility matrix. Gaussian white noise processes with mean, variance, and covariance shifts are included in the simulation to compare CCP methodology to existing methods in benchmark scenarios. 

\subsection{Simulation Results}
Simulation results and figures for dependent processes are presented in this section. Tables of results, and tables and figures for white noise processes included for comparison to existing methods, are given in the supplement.
\par
For VAR($\gamma$) processes with a change point present, the CCP method outperforms existing methodology. This advantage increases with more lagged values in the dependence structure and when process transitions to a relatively large spectral radius. Figure \ref{VARplot} displays the graphical evaluation technique defined in Equation \ref{graphcdf}. 
\begin{figure}[htb]
	\centering
		\includegraphics[width = \textwidth]{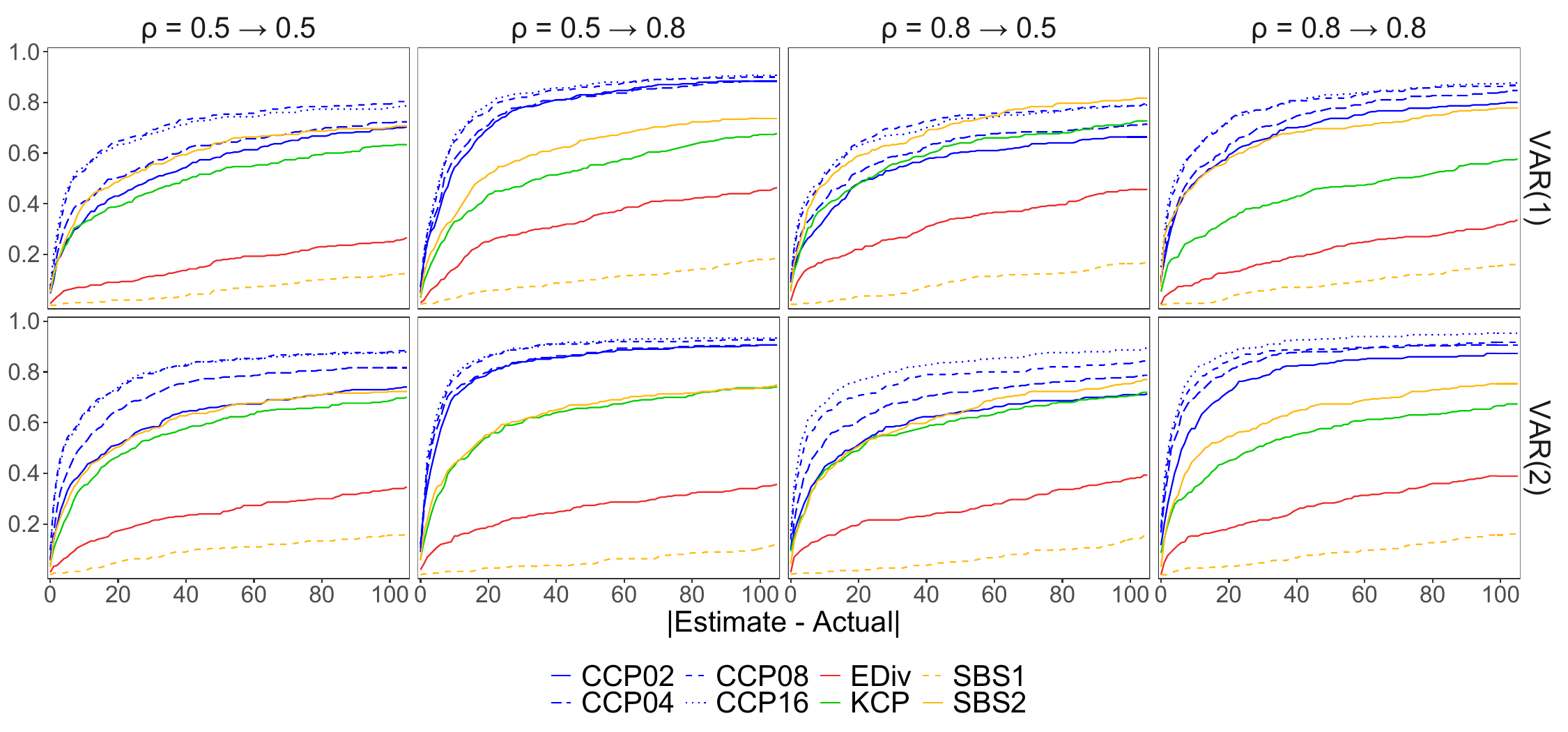}
		\caption{Fraction of identified points within error, VAR($\gamma$) simulation results with spectral radius change $\rho$, IDs (1a-d, 2a-d).}
	\label{VARplot}
\end{figure}
Among the conceptor methods, the higher error tolerances (CCP08 and CCP16) provide a more general fit to the data, where smaller error tolerances (CCP02) produce networks with larger reservoirs that learn the minor deviations of the data. In the presence of noise, these minor deviations occlude the true signal, potentially leading to inconsistencies in the learned behavior. Caution should be taken when specifying the error tolerance in noisy data; moderate tolerances may perform better than small tolerances as they fit networks with constrained internal dynamics, placing more emphasis on a general signal. This phenomenon can be likened to an overfitting problem. The KCP and SBS2 methods are the closest existing methods to the conceptor performance. One major drawback of the SBS method is that the type of change point sought must be specified; the algorithm run with a Type 1 designation produces very low ARI scores for all VAR($\gamma$) simulations.
\par
Figure \ref{PERplot} shows the results for simulations where the underlying data is generated by a periodic process. The CCP method is able to reliably detect changes in the frequency of periodic processes when the deviation from an initial state is sufficiently large; existing methodology struggles with this class of processes. While methods in the frequency domain easily detect this type of change, methods that exist in the temporal domain often fail with periodic data. The CCP method in the temporal domain is able to capture this type of nonlinear dependence, as well as those more readily described by the time axis.
\begin{figure}[htb]
	\centering
		\includegraphics[width = \textwidth]{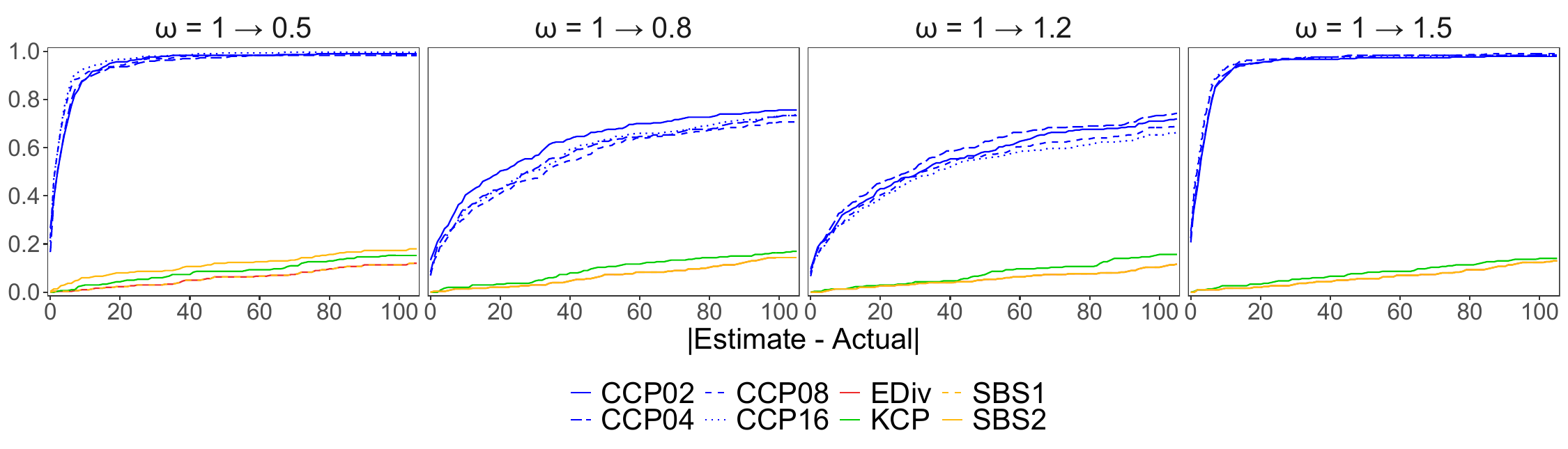}
		\caption{Fraction of identified points within error, periodic simulation results with frequency change $\omega$, IDs (3a-d).}
	\label{PERplot}
\end{figure}
\par
Figures \ref{OUplot1} and \ref{OUplot2} display results for Ornstein-Uhlenbeck simulations with a mean reverting or volatility parameter change. The CCP method surpasses most other methods for detection in mean reverting parameter changes except in some cases when the data shifts to a random walk process (or the parameter goes to zero). The difficulties for all methods can likely be attributed to a relatively low signal to noise ratio. For the volatility, the CCP method is competitive with existing methodology. This behavior is also seen in white noise variance simulations (see suplementary material) that play to the strengths of the comparator methods. With a high signal to noise ratio, the CCP method is also competitive in detecting mean changes in white noise processes.
\begin{figure}[htb]
		\centering
		\includegraphics[width = \textwidth]{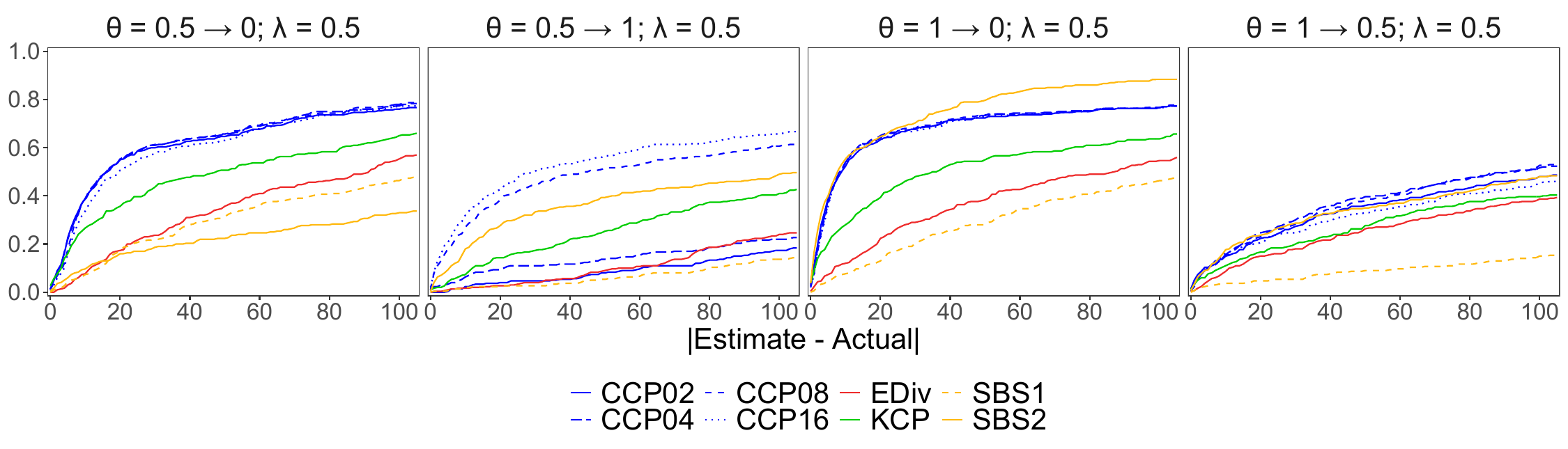}
		\caption{Fraction of identified points within error, Ornstein-Uhlenbeck simulation results with mean reverting change $\theta$, IDs (4a-d).}
		\label{OUplot1}
\end{figure}
\begin{figure}[htb]
		\centering
		\includegraphics[width = \textwidth]{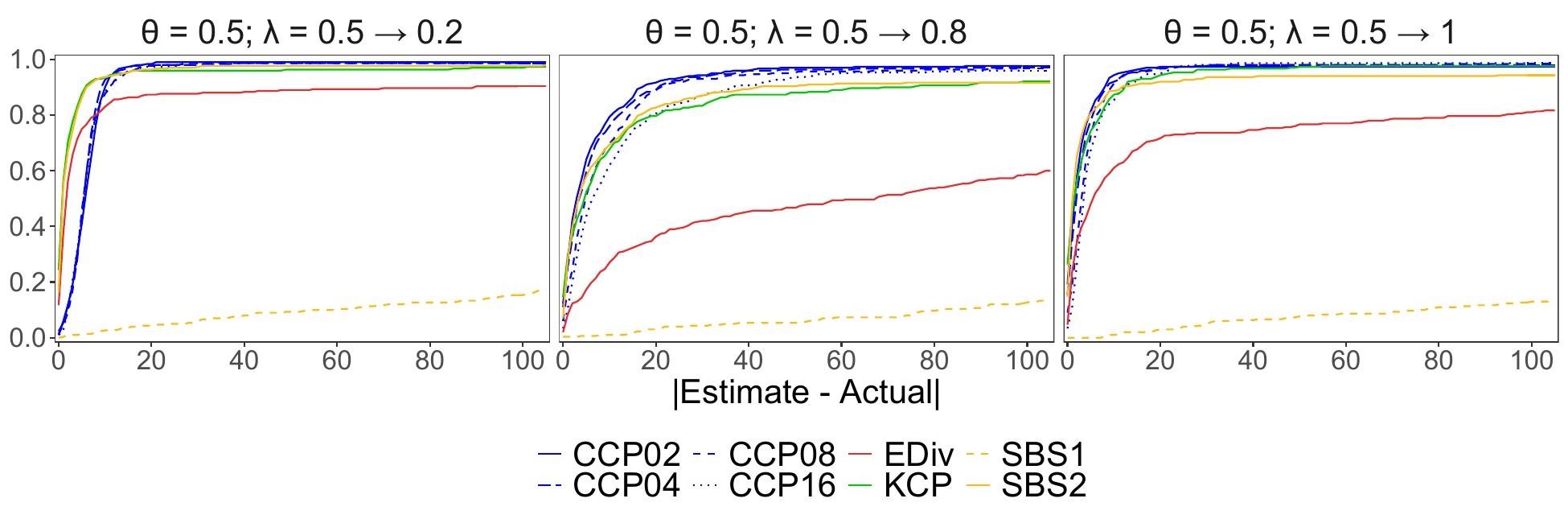}
		\caption{Fraction of identified points within error, Ornstein-Uhlenbeck simulation results with volatility change $\lambda$, IDs (4e-g).}
	\label{OUplot2}
\end{figure}
\par
To evaluate the validity of the moving block bootstrap in the CCP method, the Type 1 error of each method is observed when no change takes place in the time series. Control for false discovery of change points is as important as the correct identification of a change. Figure \ref{NCPfig} shows the observed probability of erroneous detection for each method. SBS methods provide conservative error control, KCP methodology almost always flags a change point, and EDiv does not hold to a desired level for data that is not Gaussian white noise. SBS methods do not return a quantile estimate, but only a binary ``present'' or ``not present'' flag; to estimate coverage at the points indicated, the method was run with different values of the threshold $q$. CCP methodology tracks along the uniform cdf with only slight undercoverage in periodic data and VAR data with large spectral radii.
\begin{figure}[htb]
	\centering
		\includegraphics[width = \textwidth]{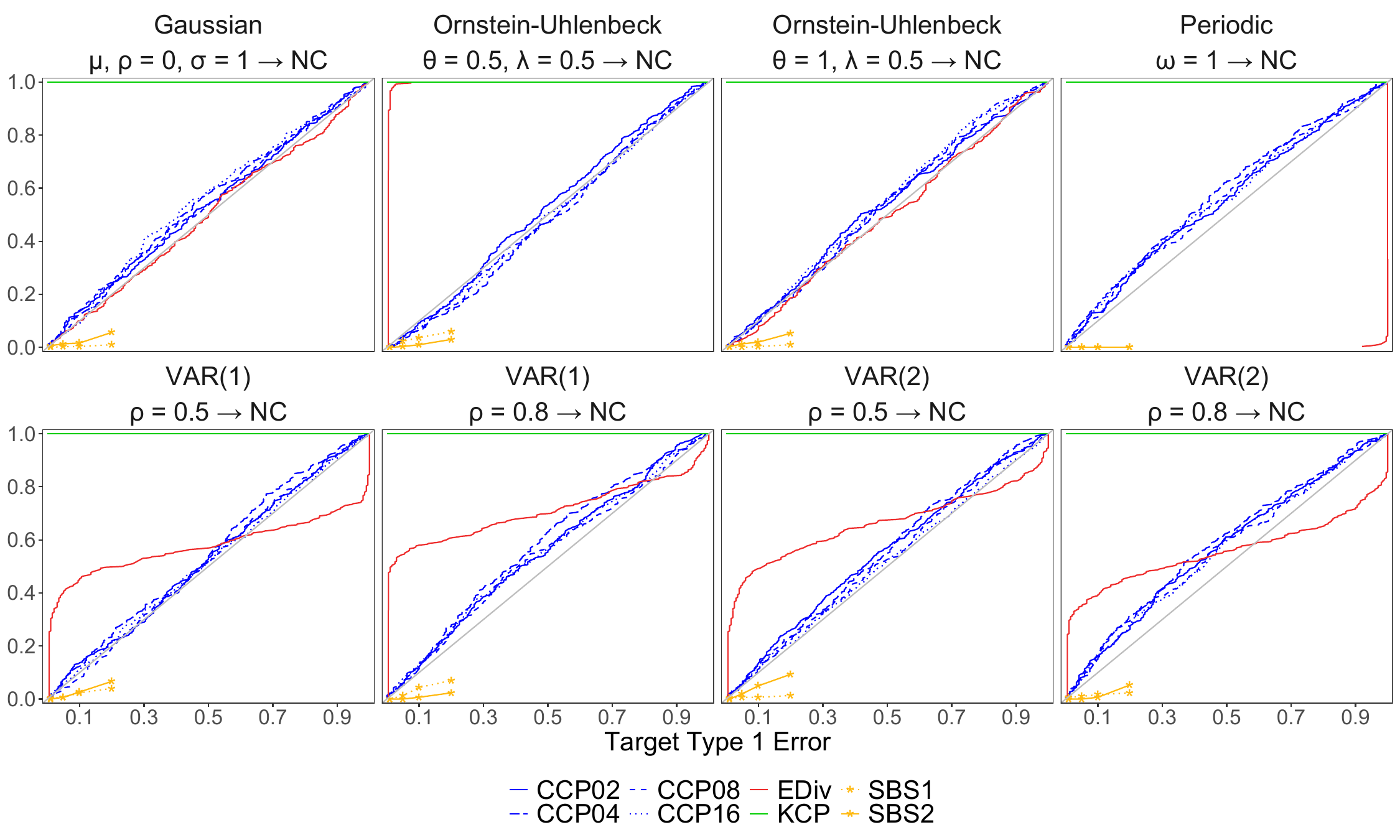}
		\caption{Type 1 error control with uniform cdf included for reference, IDs (1e-f, 2e-f, 3e, 4h-i, 5j).}
	\label{NCPfig}
\end{figure}

\section{Application Study}
\label{se:CCPapplication}
The CCP method is applied to data from \citet{varela19}. The authors record local field potential (LFP) up to 600Hz in the midline thalamus (THAL), medial prefrontal cortex (PFC), and the CA1 region of the hippocampus (HC) in rats experiencing bouts of non-REM sleep and wakefulness while they remained in a quiet, square-shaped enclosure \citep{varela19, varela19b}. The data, obtained from the Collaborative Research in Computational Neuroscience data sharing website, also includes determinations of sleep state (awake or non-REM sleep), as well as spiking, spindle, and sharp-wave ripple information over the course of the experiment \citep{varela19}. Session 1 is selected (prior to exploration of a radial maze), and the data is filtered with a finite impulse response filter to focus on the delta band (1-4Hz), characterizing slow wave sleep. Finally, the data is downsampled to a frequency of 4Hz.
\par
Three periods of transition identified by \citet{varela19} are isolated; each spans 100 seconds: sleep to wake (650 to 750s, change point at 740s), wake to sleep (740 to 840s, change point at 800s), and wake to sleep (1080 to 1180s, change point at 1150s). Relatively short windows are selected to satisfy the AMOC assumption; shorter time periods focus the methods to detecting the sleep state transition rather than other dynamic neural process changes almost certainly present in the data. The CCP, EDiv, SBS2, and KCP methods are evaluated on their ability to locate the change points in each transition period. 
\par
Approximately 10 seconds are reserved for reservoir washout, and 30 seconds are used for conceptor training with the CCP method. Change points are identified in the remaining one minute for all methods. Settings for methods are given in Table \ref{CCPsimsettings}, and the CCP error tolerance is kept at $\varepsilon_\text{train}=4\%$ NRMSE. Figure \ref{CCPexample} presents the results from applying the methods to the LFP data. Methods that fail to identify a change point display as a vertical line at the far right edge of the figure. 
\begin{figure}[htb]
		\centering
		\includegraphics[width = \textwidth]{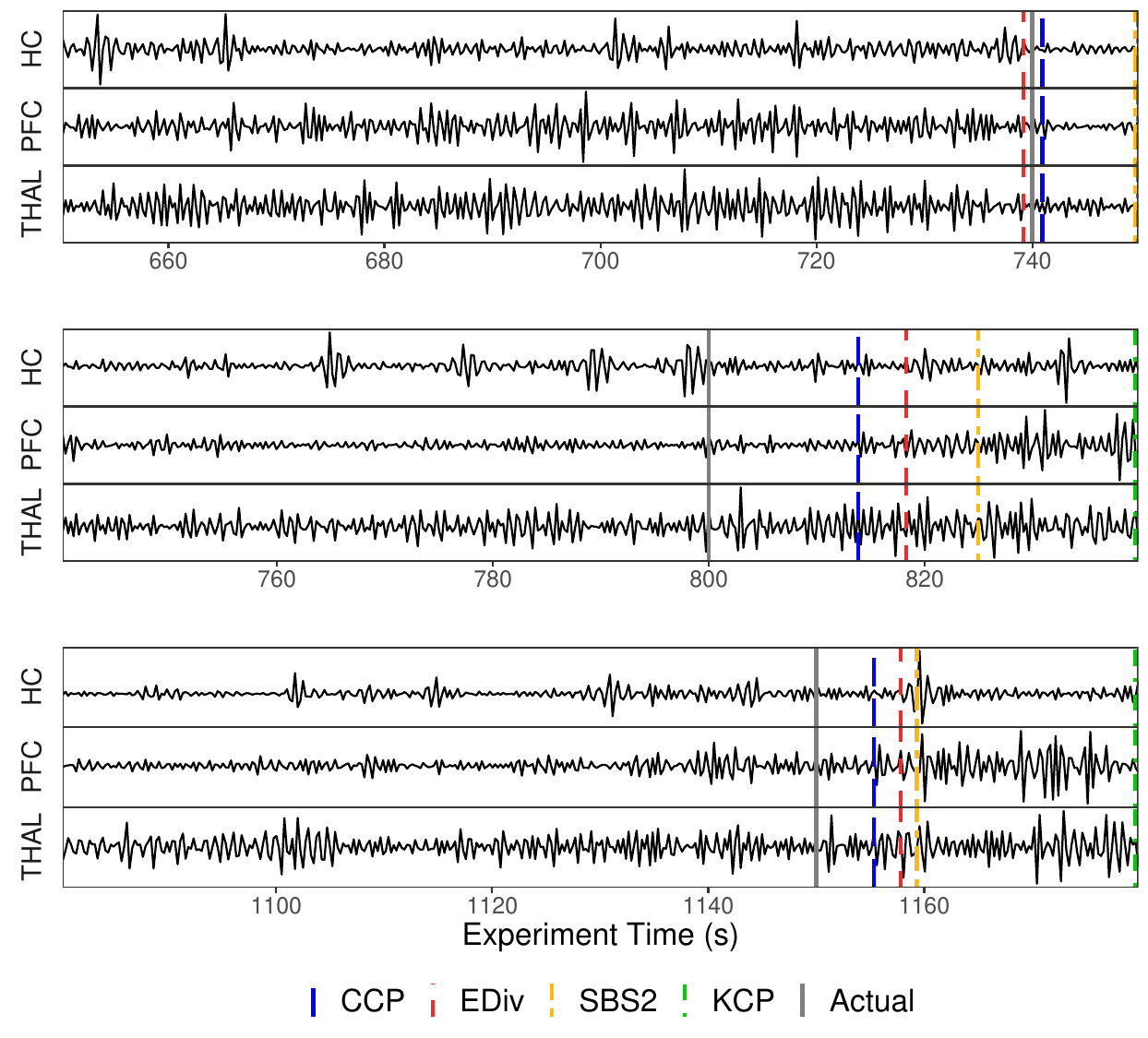}
		\caption{Estimated change points in LFP examples with actual points identified in \citet{varela19}. \textit{Top}: Sleep to Wake, \textit{Middle \& Bottom}: Wake to Sleep. }
	\label{CCPexample}
\end{figure}
\par
Figure \ref{CCPexample1} displays the internal dynamics of the conceptor methodology applied to the first (top) time series segment in Figure \ref{CCPexample}. Similar visuals of the second and third segments are shown in the supplement. 
\begin{figure}[htb]
		\centering
		\includegraphics[width = \textwidth]{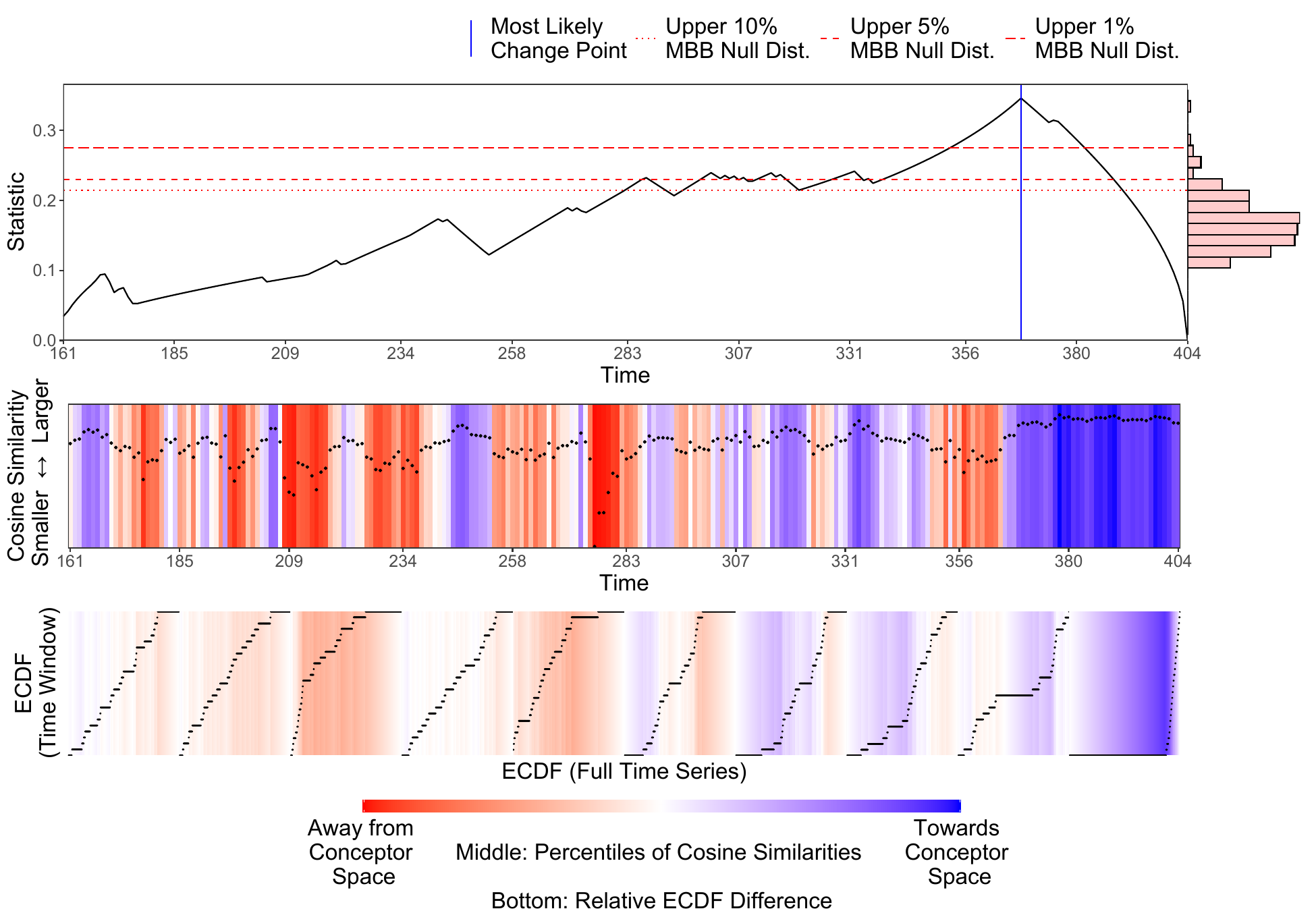}
		\caption{CCP method visualization of Figure \ref{CCPexample} (\textit{top}). Proposed Change Point: 368 (740.9s), Statistic = 0.346, MBB quantile = 0. \textit{Top}: Identification of most likely change point from CUSUM-like statistics. Null bootstrap distribution included on the right vertical axis with estimated quantiles. \textit{Middle}: Cosine similarities between conceptor and reservoir space over the time series. Shading represents percentiles of cosine similarities over the full time series. \textit{Bottom}: Compares segment specific cosine similarity empirical CDFs to the full time series. Shading represents a relative difference of empirical CDFs.}
	\label{CCPexample1}
\end{figure}
The top plot of Figure \ref{CCPexample1} displays the series of CUSUM-like statistics, with an estimated bootstrap null distribution on the right vertical axis and quantiles as horizontal lines on the plot. The middle plot displays $S_t$ with a relative vertical axis, as only comparative differences through the time series are sought. The bottom plot gives empirical CDFs of $S_t$ over segmented windows of time; the plotted points display the difference between the empirical CDF of the specific window and the overall empirical CDF of the full time series. Shading in the middle and bottom plots represents the internal reservoir dynamics and their relationship to the conceptor space; blue refers to dynamics that are behaving similarly to the original conceptor space, and red indicates further away. The scale of color shading in the middle plot is tied to percentiles the sequence $S_t$, and in the bottom plot is a relative difference between empirical CDFs. Change points will be identified as a peak in the top plot and a color transition in the middle and bottom plots. Excessive undulation, a secondary peak, or multiple shifts may suggest violation of the AMOC assumption or a slow transition between states.

\section{Discussion}
The CCP method provides a model agnostic framework for addressing nonlinear temporal dependence in change point identification problems. This relaxes the common i.i.d. assumption of most existing methodology, and allows for flexible definition of a baseline state without the rigidity of an imposed structure. The method also alleviates the problem of specifying a functional nonlinear form, which can be challenging. 
\par
The ESN learns the characteristic dynamics of a training window, and the deviation from the conceptor space is examined with a CUSUM-like statistic that consistently estimates the true change point under mild assumptions. The method is able to flag important locations for future scrutiny and provide guidance on the strength of evidence for a change point via the moving block bootstrap. CCP outperforms existing methods in temporally dependent and periodic processes, and is even competitive in i.i.d. processes with a change in variation or a high signal to noise ratio. In practice, the training window should be sufficiently long to capture representative variation of the original time series, and $\varepsilon_\text{train}$ left at a default value unless there is prior knowledge of the type of change sought. Assumptions include a baseline period of stable data generation where the data is at least wide-sense cyclostationary, and stationarity of the obtained similarity sequence $S_t$ on either side of at most one potential change point . Violation of the cosine similarity stationarity assumption will affect theoretical results, but does not diminish use of the method for investigative study of a dataset. Implementation of the method may require isolation of a time segment of interest, like the application in Section \ref{se:CCPapplication}, so that the AMOC assumption is met. These segments must be identified by the researcher with prior knowledge of their data. This method provides straightforward extensions to the multiple change point problem, and to online, sequential detection provided changes are sparse and sufficiently spaced. Future work can improve the featurization process so that information pertaining to individual series in the data is preserved, making qualitative conclusions about change points more accessible.

\bibliography{CCP_arXiv}
	
	\FloatBarrier
	\newpage
	\FloatBarrier
	\begin{center}
		{\large\bf SUPPLEMENTARY MATERIAL}\\
		{\normalsize \bf Change Point Detection with Conceptors}
	\end{center}
	\begin{description}
		\item[Additional Tables \& Figures:] Tables and figures of simulation and application results not presented in the main paper.
		\item[Algorithms \& Pseudocode:] Presentation of algorithms and procedures mentioned in Section \ref{se:CCPmethods} in the form of pseudocode. See main text for explanation.
		\item[Proofs:] Proofs to Theorems presented in Section \ref{se:CCPtheory} of the paper.
 		\item[Code \& Supporting Material:] Files (.RData) used to assess performance of change point methods, code (.R) used to generate results and figures, and data (.mat and .csv) and code (.R) used to generate output in Section \ref{se:CCPapplication} can be found at $$\verb|github.com/noahgade/ChangePointDetectionWithConceptors|.$$
	\end{description}
	\FloatBarrier
	\newpage
	\FloatBarrier
	\section*{Additional Tables \& Figures}
	\subsection*{Tables}
	\begin{table}[htb]
		\centering
			\begin{tabular}{ccccccccc}
				\hline
				& CCP02 & CCP04 & CCP08 & CCP16 & EDiv & KCP & SBS1 & SBS2 \\ 
				\hline
				$\rho = 0.5 \rightarrow 0.5$ & 0.576 & 0.620 & \bf 0.716 & 0.700 & 0.192 & 0.556 & 0.011 & 0.584 \\ 
				$\rho = 0.5 \rightarrow 0.8$ & 0.755 & 0.751 & 0.784 & \bf 0.795 & 0.385 & 0.600 & 0.029 & 0.606 \\ 
				$\rho = 0.8 \rightarrow 0.5$ & 0.549 & 0.609 & \bf 0.698 & 0.682 & 0.393 & 0.668 & 0.044 & 0.679 \\ 
				$\rho = 0.8 \rightarrow 0.8$ & 0.685 & 0.719 & 0.766 & \bf 0.775 & 0.288 & 0.523 & 0.032 & 0.656 \\ 
				\hline
			\end{tabular}
			\caption{Mean ARI of VAR(1) spectral radius $\rho$ changes, IDs (1a-d).}
		\label{VAR1sims}
	\end{table}
	
	\begin{table}[htb]
			\centering
			\begin{tabular}{ccccccccc}
				\hline
				& CCP02 & CCP04 & CCP08 & CCP16 & EDiv & KCP & SBS1 & SBS2 \\ 
				\hline
				$\rho = 0.5 \rightarrow 0.5$ & 0.595 & 0.694 & 0.774 & \bf 0.769 & 0.239 & 0.607 & 0.012 & 0.575 \\ 
				$\rho = 0.5 \rightarrow 0.8$ & 0.837 & 0.848 & 0.873 & \bf 0.883 & 0.293 & 0.677 & 0.012 & 0.654 \\ 
				$\rho = 0.8 \rightarrow 0.5$ & 0.610 & 0.697 & 0.769 & \bf 0.816 & 0.302 & 0.655 & 0.016 & 0.622 \\ 
				$\rho = 0.8 \rightarrow 0.8$ & 0.776 & 0.822 & 0.839 &\bf  0.880 & 0.310 & 0.606 & 0.027 & 0.621 \\ 
				\hline
			\end{tabular}
			\caption{Mean ARI of VAR(2) spectral radius $\rho$ changes, IDs (2a-d).}
		\label{VAR2sims}
	\end{table}
	
	\begin{table}[htb]
			\centering
			\begin{tabular}{ccccccccc}
				\hline
				& CCP02 & CCP04 & CCP08 & CCP16 & EDiv & KCP & SBS1 & SBS2 \\ 
				\hline
				$\omega = 1 \rightarrow 0.5$ & 0.943 & 0.936 & 0.953 & \bf 0.961 & 0.000 & 0.057 & 0.000 & 0.057 \\ 
				$\omega = 1 \rightarrow 0.8$ & \bf 0.649 & 0.597 & 0.559 & 0.586 & 0.000 & 0.056 & 0.000 & 0.000 \\ 
				$\omega = 1 \rightarrow 1.2$ & 0.587 & \bf 0.605 & 0.538 & 0.532 & 0.000 & 0.056 & 0.000 & 0.000 \\ 
				$\omega = 1 \rightarrow 1.5$ & 0.932 & 0.935 & 0.940 & \bf 0.941 & 0.000 & 0.055 & 0.000 & 0.000 \\ 
				\hline
			\end{tabular}
			\caption{Mean ARI of periodic frequency $\omega$ changes, IDs (3a-d).}
		\label{PERsims}
	\end{table}
	
	\begin{table}[htb]
			\centering
			\begin{tabular}{ccccccccc}
				\hline
				& CCP02 & CCP04 & CCP08 & CCP16 & EDiv & KCP & SBS1 & SBS2 \\ 
				\hline
				$\theta = 0.5 \rightarrow 0$ & 0.709 & \bf 0.718 & 0.716 & 0.701 & 0.516 & 0.633 & 0.419 & 0.199 \\ 
				$\theta = 0.5 \rightarrow 1$ & 0.056 & 0.124 & 0.479 & \bf 0.538 & 0.241 & 0.359 & 0.018 & 0.368 \\ 
				$\theta = 1 \rightarrow 0$ & 0.738 & \bf 0.743 & 0.738 & 0.736 & 0.523 & 0.633 & 0.406 & 0.720 \\ 
				$\theta = 1 \rightarrow 0.5$ & 0.328 & \bf 0.367 & 0.365 & 0.311 & 0.304 & 0.348 & 0.020 & 0.313 \\ 
				\hline
			\end{tabular}
			\caption{Mean ARI of Ornstein-Uhlenbeck mean reverting $\theta$ changes, IDs (4a-d).}
			\label{OUsims1}
\end{table}
\begin{table}[htb]
	\centering
			\begin{tabular}{ccccccccc}
				\hline
				& CCP02 & CCP04 & CCP08 & CCP16 & EDiv & KCP & SBS1 & SBS2 \\ 
				\hline
				$\lambda = 0.5 \rightarrow 0.2$ & 0.922 & 0.924 & 0.924 & 0.926 & 0.874 & \bf 0.949 & 0.053 & 0.927 \\ 
				$\lambda = 0.5 \rightarrow 0.8$ & \bf 0.909 & 0.896 & 0.882 & 0.870 & 0.558 & 0.861 & 0.005 & 0.836 \\ 
				$\lambda = 0.5 \rightarrow 1$ & 0.936 & 0.935 & 0.934 & 0.931 & 0.767 & \bf 0.942 & 0.016 & 0.886 \\
				\hline
			\end{tabular}
			\caption{Mean ARI of Ornstein-Uhlenbeck volatility $\lambda$ changes, IDs (4e-g).}
		\label{OUsims2}
	\end{table}
	
	\begin{table}[htb]
			\centering
			\begin{tabular}{ccccccccc}
				\hline
				& CCP02 & CCP04 & CCP08 & CCP16 & EDiv & KCP & SBS1 & SBS2 \\ 
				\hline
				$\mu = 0 \rightarrow 0.5$ & 0.478 & 0.532 & 0.620 & 0.638 & 0.868 & \bf 0.904 & 0.802 & 0.000 \\ 
				$\mu = 0 \rightarrow 0.8$ & 0.840 & 0.846 & 0.850 & 0.845 & 0.922 & \bf 0.962 & 0.899 & 0.002 \\ 
				$\mu = 0 \rightarrow 1$ & 0.899 & 0.894 & 0.897 & 0.898 & 0.940 & \bf 0.974 & 0.898 & 0.000 \\ 
				\hline
			\end{tabular}
			\caption{Mean ARI of white noise mean $\mu$ changes, IDs (5a-c).}
			\label{WNsims1}
\end{table}
		\begin{table}[htb]
			\centering
			\begin{tabular}{ccccccccc}
				\hline
				& CCP02 & CCP04 & CCP08 & CCP16 & EDiv & KCP & SBS1 & SBS2 \\ 
				\hline
				$\sigma = 1 \rightarrow 0.5$ & 0.921 & 0.916 & 0.905 & 0.905 & 0.883 & \bf 0.964 & 0.011 & 0.866 \\ 
				$\sigma = 1 \rightarrow 0.8$ & \bf 0.720 & 0.718 & 0.578 & 0.527 & 0.039 & 0.716 & 0.007 & 0.345 \\ 
				$\sigma = 1 \rightarrow 1.2$ & \bf 0.631 & 0.604 & 0.503 & 0.409 & 0.032 & 0.599 & 0.001 & 0.214 \\ 
				$\sigma = 1 \rightarrow 1.5$ & 0.917 & 0.903 & 0.891 & 0.872 & 0.690 & \bf 0.920 & 0.005 & 0.796 \\ 
				$\rho = 0 \rightarrow 0.8$ & 0.034 & 0.040 & 0.082 & 0.087 & 0.299 & \bf 0.916 & 0.000 & 0.857 \\ 
				\hline
			\end{tabular}
			\caption{Mean ARI of white noise variance $\sigma$ and covariance $\rho$ changes, IDs (5d-h).}
		\label{WNsims2}
	\end{table}
	
	\begin{table}[htb]
			\centering
			\begin{tabular}{ccccccccccc}
				\hline
				ID & Class & CCP02 & CCP04 & CCP08 & CCP16 & EDiv & KCP & SBS1 & SBS2 \\ 
				\hline
				(1e) & VAR(1) & 0.07 & 0.07 & 0.04 & 0.06 & 0.41 & 1.00 & 0.01 & 0.01 \\ 
				(1f) & VAR(1) & 0.06 & 0.07 & 0.05 & 0.08 & 0.54 & 1.00 & 0.01 & 0.00 \\ 
				(2e) & VAR(2) & 0.06 & 0.06 & 0.07 & 0.06 & 0.44 & 1.00 & 0.01 & 0.02 \\ 
				(2f) & VAR(2) & 0.09 & 0.08 & 0.08 & 0.08 & 0.35 & 1.00 & 0.01 & 0.00 \\ 
				(3e) & Periodic & 0.07 & 0.09 & 0.06 & 0.09 & 0.00 & 1.00 & 0.00 & 0.00 \\ 
				(4h) & Orn.-Uhl. & 0.04 & 0.04 & 0.03 & 0.04 & 0.99 & 1.00 & 0.02 & 0.00 \\ 
				(4i) & Orn.-Uhl. & 0.05 & 0.05 & 0.04 & 0.05 & 0.03 & 1.00 & 0.00 & 0.01 \\ 
				(5i) & Wh. Noise & 0.06 & 0.08 & 0.05 & 0.05 & 0.05 & 1.00 & 0.00 & 0.01 \\ 
				\hline
			\end{tabular}
			\caption{Observed Type 1 error for $q=0.05$ with no change point.}
		\label{NCPsims}
	\end{table}
	
	\FloatBarrier
	\subsection*{Figures}
	\begin{figure}[htb]
			\centering
			\includegraphics[width = \textwidth]{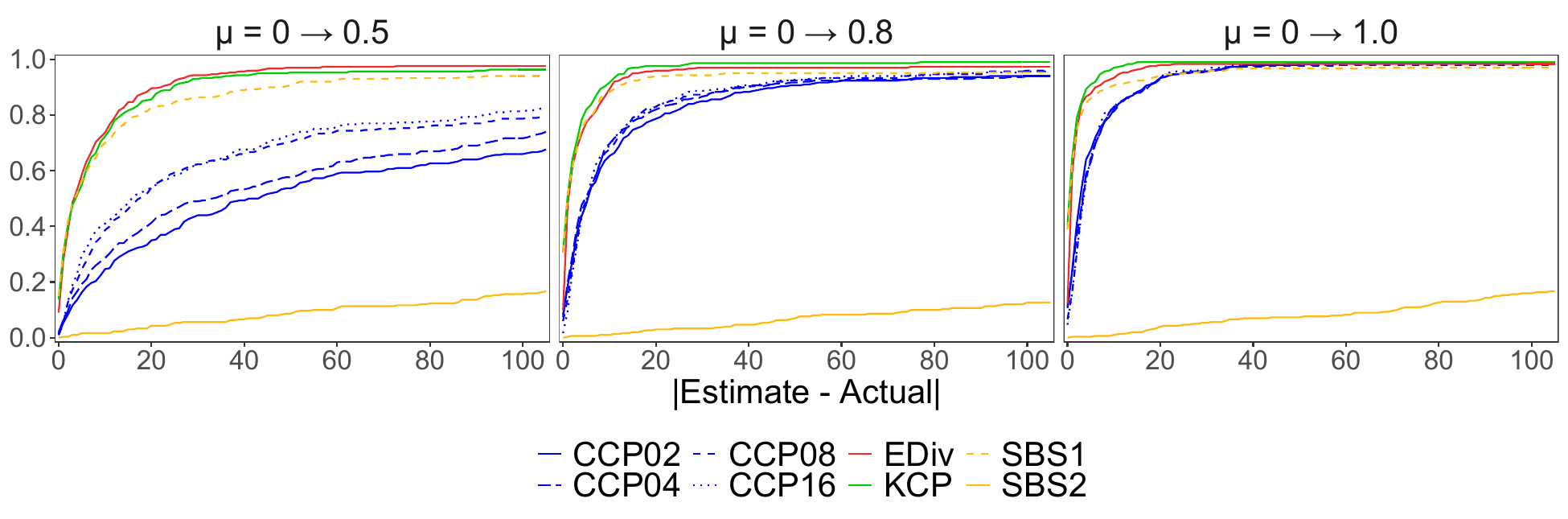}
			\caption{Fraction of identified points within error, white noise simulation results with mean change $\mu$, IDs (5a-c).}
			\label{WNplot1}
	\end{figure}

	\begin{figure}[htb]
			\centering
			\includegraphics[width = \textwidth]{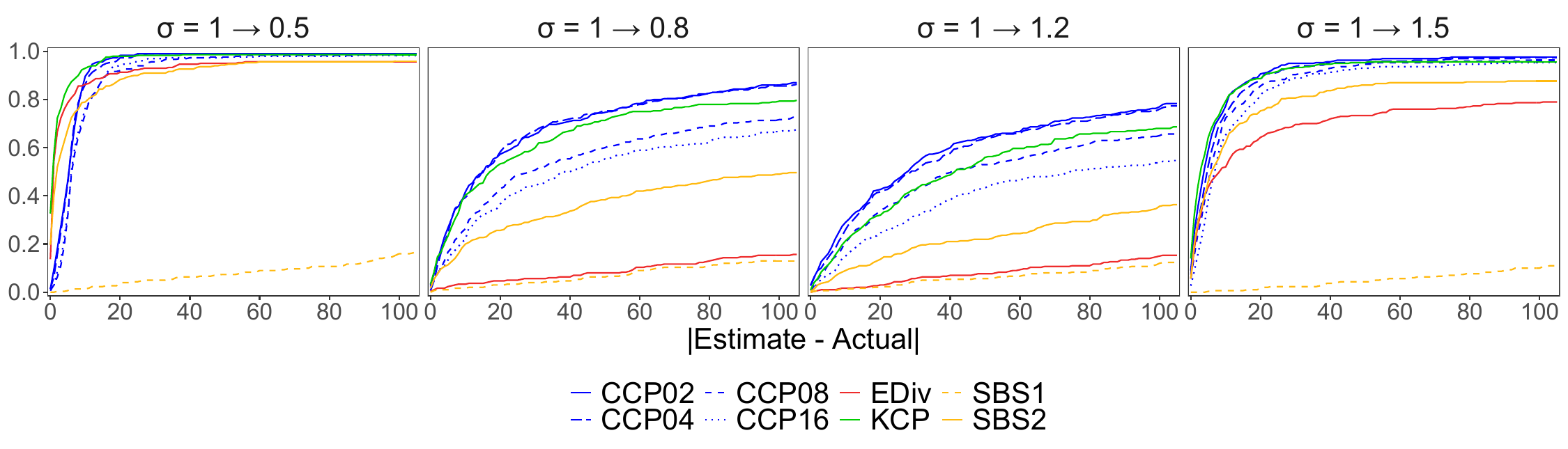}
			\caption{Fraction of identified points within error, white noise simulation results with variance change $\sigma$, IDs (5d-g).}
			\label{WNplot2}
		\end{figure}

	\begin{figure}[htb]
			\centering
			\includegraphics[width = 0.75\textwidth]{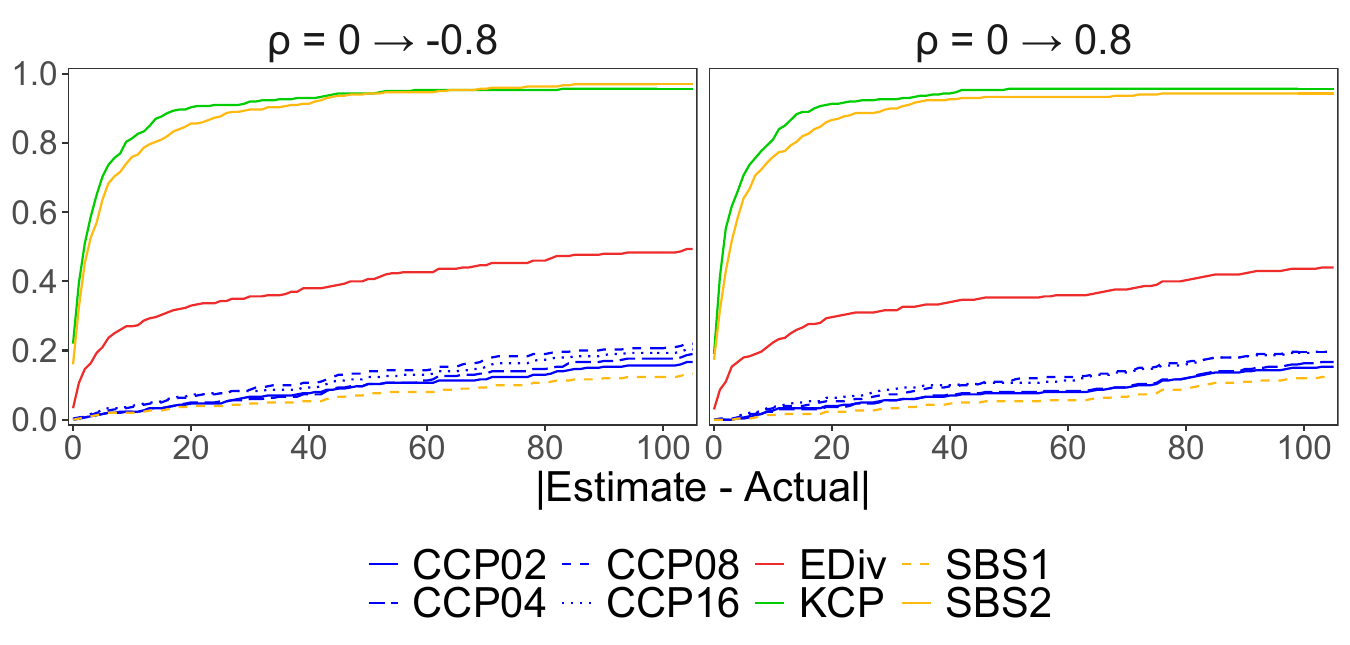}
			\caption{Fraction of identified points within error, white noise simulation results with covariance change $\rho$, ID (5h).}
		\label{WNplot3}
	\end{figure}
	
	\begin{figure}[htb]
			\centering
			\includegraphics[width = \textwidth]{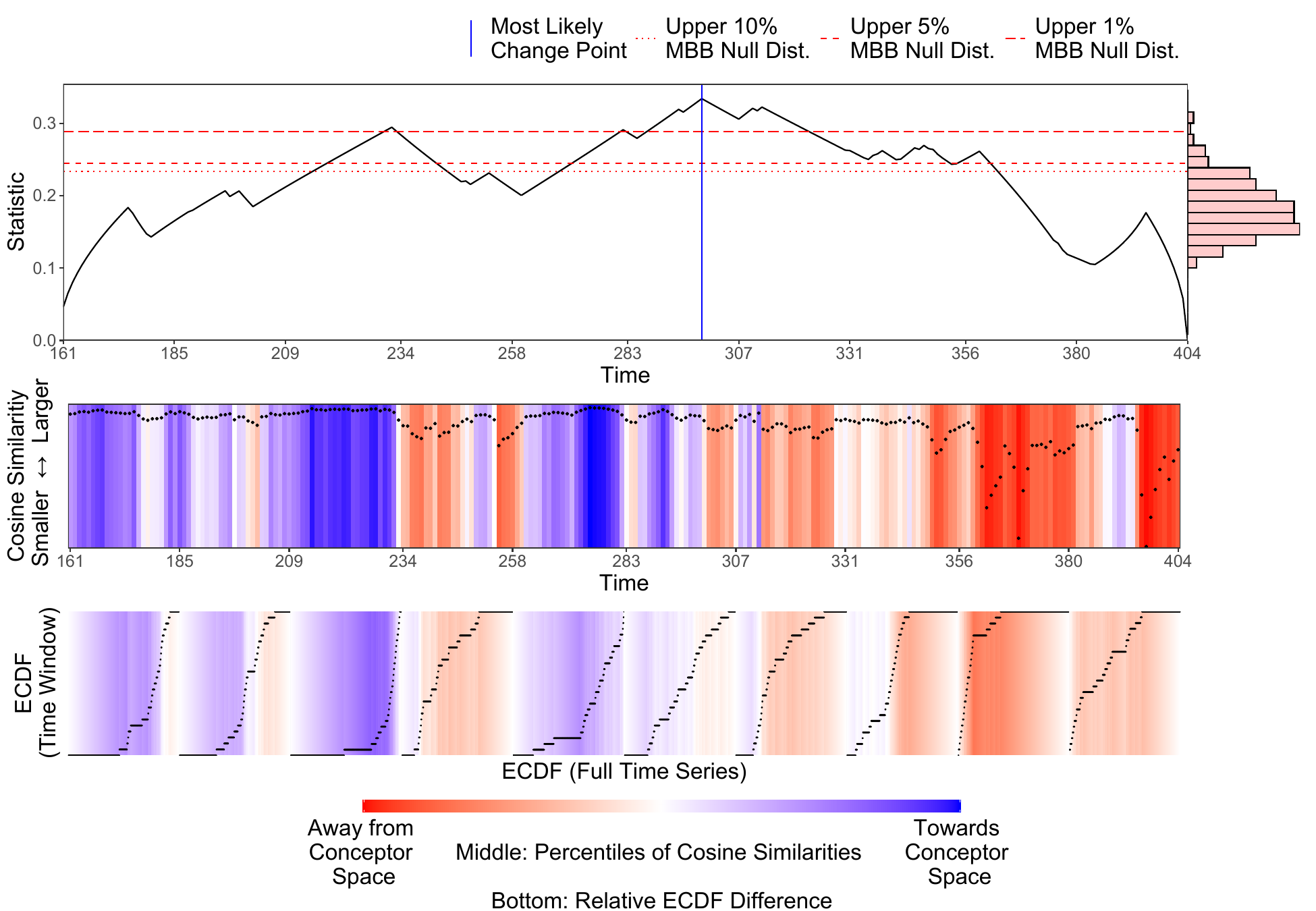}
			\caption{CCP method visualization of Figure \ref{CCPexample} (\textit{middle}). Proposed Change Point: 299 (813.8s), Statistic = 0.334, MBB quantile = 0. \textit{Top}: Identification of most likely change point from CUSUM-like statistics. Null bootstrap distribution included on the right vertical axis with estimated quantiles. \textit{Middle}: Cosine similarities between conceptor and reservoir space over the time series. Shading represents percentiles of cosine similarities over the full time series. \textit{Bottom}: Compares segment specific cosine similarity empirical CDFs to the full time series. Shading represents a relative difference of empirical CDFs.}
		\label{CCPexample2}
	\end{figure}
	
	\begin{figure}[htb]
			\centering
			\includegraphics[width = \textwidth]{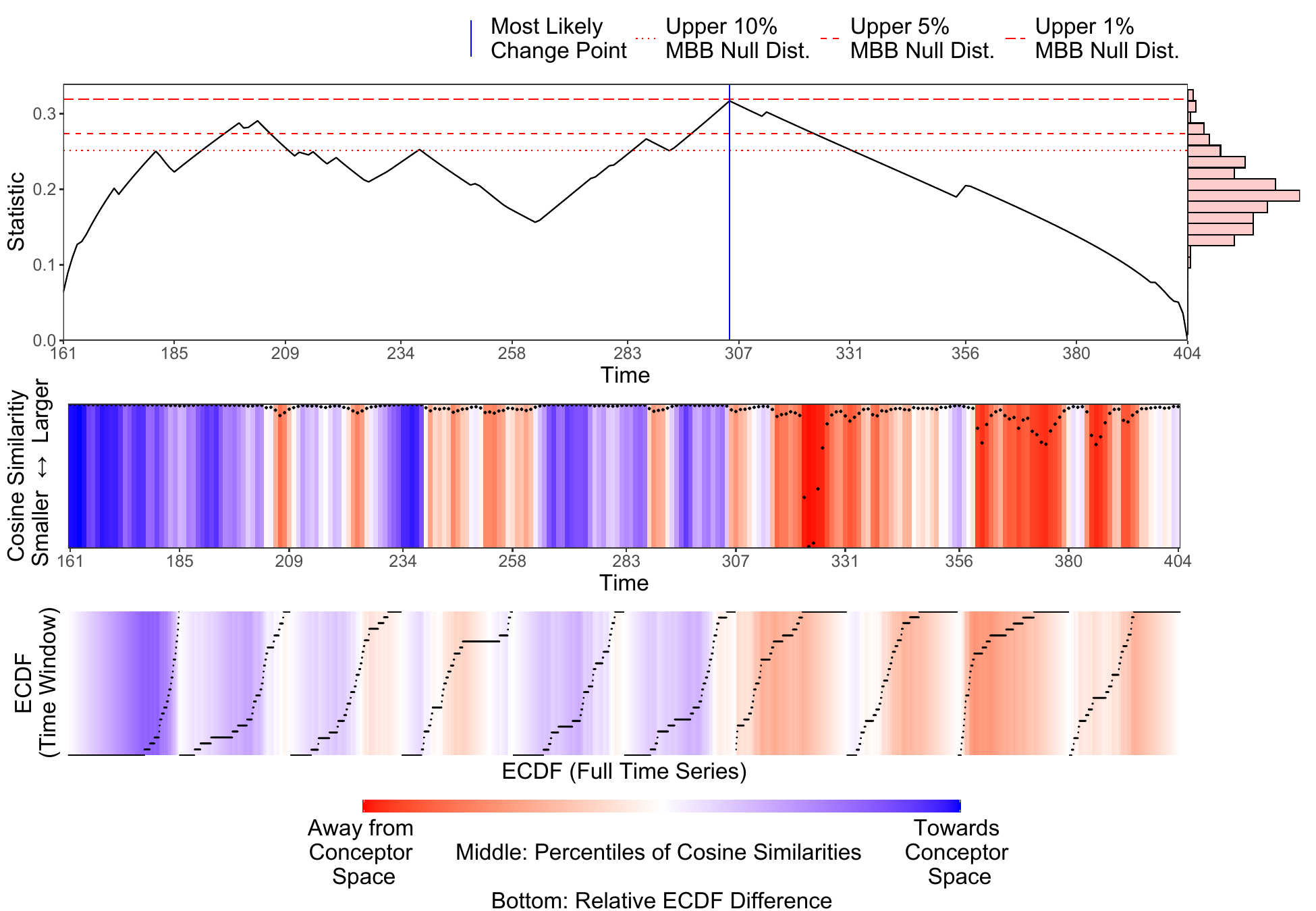}
			\caption*{CCP method visualization of Figure \ref{CCPexample} (\textit{bottom}). Proposed Change Point: 305 (1155.4s), Statistic = 0.317, MBB quantile = 0.013. \textit{Top}: Identification of most likely change point from CUSUM-like statistics. Null bootstrap distribution included on the right vertical axis with estimated quantiles. \textit{Middle}: Cosine similarities between conceptor and reservoir space over the time series. Shading represents percentiles of cosine similarities over the full time series. \textit{Bottom}: Compares segment specific cosine similarity empirical CDFs to the full time series. Shading represents a relative difference of empirical CDFs.}
		\label{CCPexample3}
	\end{figure} 
	
	\FloatBarrier
	\newpage
	\FloatBarrier
	\section*{Algorithms \& Pseudocode}
	\floatname{algorithm}{Algorithm}
	\begin{algorithm}
		\caption{ESN Featurization}
		\label{algorithm1}
		\vskip0in
		\textbf{Inputs}: time series data $\mathbf{y}_t \in \mathbb{R}^d$; training window length $T_\text{train}$
		\vskip0in
		\textbf{Outputs}: ESN parameters; ESN network size $N$; aperture $\alpha$; washout length $T_\text{wash}$
		\vskip0in
		\begin{algorithmic}[1]
			\State perform Procedure \ref{procedure1} to obtain $\text{ESN scaling}: \left\{c_\text{input}, c_\text{bias}, \rho\right\}$
			
			\State perform Procedure \ref{procedure3} to obtain ESN reservoir size $N$ and aperture $\alpha$
			
			\Return $\text{ESN scaling}: \left\{c_\text{input}, c_\text{bias}, \rho\right\}$, $N$, $\alpha$, $T_\text{wash}$
		\end{algorithmic}
	\end{algorithm}
	
	\floatname{algorithm}{Algorithm}
	\begin{algorithm}
		\caption{Change Point Proposal}
		\label{algorithm2}
		\vskip0in
		\textbf{Inputs}: time series data $\mathbf{y}_t \in \mathbb{R}^d$; training window length $T_\text{train}$; ESN scaling: $\left\{c_\text{input}, c_\text{bias}, \rho\right\}$, $N$, $\alpha$, $T_\text{wash}$ from Algorithm \ref{algorithm1}
		\vskip0in
		\textbf{Outputs}: most likely change point $\hat{\tau}$; statistic $K$
		\vskip0in
		\textbf{Default Parameters}: number of featurizations $\mathscr{R} \leftarrow 100$
		\vskip0in
		\begin{algorithmic}[1]
			\For{$r$ in $1:\mathscr{R}$}
			
			\State initialize $\mathbf{W^i}_r,\mathbf{b}_r,\mathbf{W^h}_r$ where each element $\mathcal{N}(0,1)$, and $\mathbf{W^h}_r$ is sparse
			
			\State $\mathbf{W^i}_r\leftarrow c_\text{input}\mathbf{W^i}_r$;\hskip0.1in $\mathbf{b}_r\leftarrow c_\text{bias}\mathbf{b}_r$
			
			\State $\mathbf{W^h}_r\leftarrow \rho\mathbf{W^h}_r/\max\left\{\mathbf{v}'\mathbf{W^h}_r\mathbf{v}: ||\mathbf{v}||=1\right\}$
			
			\State $\mathbf{h}_{r,t} \leftarrow \tanh\left(\mathbf{W^h}_r \mathbf{h}_{r,t-1} + \mathbf{W^i}_r \mathbf{y}_{t} + \mathbf{b}_r\right)$ \hskip0.05in for $t=1, \ldots, T_\text{wash}+T_\text{train}$ 
			
			\State $\mathbf{C}_r \leftarrow \frac{1}{T_\text{train}}\mathbf{H}_r\mathbf{H}_r'\left(\frac{1}{T_\text{train}}\mathbf{H}_r\mathbf{H}_r' + \alpha^{-2}\mathbf{I}\right)^{-1}$ 
			\Statex \hspace*{5cm}where $\mathbf{H}_r = \left[\mathbf{h}_{r,T_\text{wash}+1} \;\cdots\;\mathbf{h}_{r,T_\text{wash}+T_\text{train}}\right]$
			
			\State $\mathbf{h}_{r,t} \leftarrow \tanh\left(\mathbf{W^h}_r \tilde{\mathbf{h}}_{r,t-1} + \mathbf{W^i}_r \mathbf{y}_{t} + \mathbf{b}_r\right)$; \hskip0.05in $\tilde{\mathbf{h}}_{r,t} \leftarrow \mathbf{C}_r\mathbf{h}_{r,t}$ for $t=T_\text{wash} + 1, \ldots, T$
			
			\State  $s_{r,t} \leftarrow \frac{\tilde{\mathbf{h}}_{r,t}' \mathbf{h}_{r,t}}{\left|\left|\tilde{\mathbf{h}}_{r,t} \right|\right|\left|\left|\mathbf{h}_{r,t}\right|\right|}$ \hskip0.05in for $t = T_\text{wash}+T_\text{train}+1,\ldots,T$
			
			\EndFor
			
			\State $S_t \leftarrow {\displaystyle \frac{1}{\mathscr{R}}\sum_{r=1}^\mathscr{R} s_{r, t}}$ \hskip0.05in for $t = T_\text{wash}+T_\text{train}+1,\ldots,T$
			
			\For{$t$ in $(T_\text{wash} + T_\text{train} + 1):(T-1)$}
			
			\State $\hat{\mathcal{F}}_{(T_\text{wash}+T_\text{train}+1):t}(s) \leftarrow {\displaystyle \frac{1}{t - T_\text{wash} - T_\text{train}}\sum_{i=T_\text{wash} + T_\text{train} + 1}^{t}} \mathbf{1}\left\{S_i \leq s\right\}$
			
			\State $\hat{\mathcal{F}}_{(t+1):T}(s) \leftarrow {\displaystyle \frac{1}{T - t}\sum_{i=t + 1}^{T}} \mathbf{1}\left\{S_i \leq s\right\}$
			
			\State $K_t \leftarrow \frac{\left(t-T_\text{wash}-T_\text{train}\right)\left(T - t\right)}{q\left(t\right)\left(T-T_\text{wash}-T_\text{train}\right)^{2}} \;{\displaystyle \sup_s} \left|\hat{\mathcal{F}}_{(T_\text{wash}+T_\text{train}+1):t}(s) - \hat{\mathcal{F}}_{(t+1):T}(s)\right|$
			
			\EndFor
			
			\State $K \leftarrow {\displaystyle \max_{t}} K_t$
			
			\State $\hat{\tau} \leftarrow {\displaystyle \arg \max_{t}} K_t$
			
			\Return $K, \hat{\tau}$, all $\mathbf{C}_r, \mathbf{W^i}_r$, $\mathbf{b}_r$, $\mathbf{W^h}_r$
		\end{algorithmic}
	\end{algorithm}
	
	\floatname{algorithm}{Algorithm}
	\begin{algorithm}
		\caption{Null Distribution Estimate via Moving Block Bootstrap}
		\label{algorithm3}
		\vskip0in
		\textbf{Inputs}: training window length $T_\text{train}$; $T_\text{wash}$ from Algorithm \ref{algorithm1}; $K$, all $\mathbf{C}_r, \mathbf{W^i}_r$, $\mathbf{b}_r$, $\mathbf{W^h}_r$, $\mathscr{R}$ from Algorithm \ref{algorithm2}
		\vskip0in
		\textbf{Outputs}: null distribution estimate of statistic $K_b$; quantile estimate at a defined Type 1 error threshold $p$
		\vskip0in
		\begin{algorithmic}[1]
			\State perform Procedure \ref{procedure4} to obtain bootstrapped data $\mathbf{y}_{b,t}$ and $B$
			
			\For{$b$ in $1:B$}
			
			\For{$r$ in $1:\mathscr{R}$}
			
			\State $\mathbf{h}_{b,r,t} \leftarrow \tanh\left(\mathbf{W^h}_r \tilde{\mathbf{h}}_{b,r,t-1} + \mathbf{W^i}_r \mathbf{y}_{b,t} + \mathbf{b}_r\right)$; \hskip0.05in $\tilde{\mathbf{h}}_{b,r,t} \leftarrow \mathbf{C}_r\mathbf{h}_{b,r,t}$ 
			
			\State  $s_{b,r,t} \leftarrow \frac{\tilde{\mathbf{h}}_{b,r,t}' \mathbf{h}_{b,r,t}}{\left|\left|\tilde{\mathbf{h}}_{b,r,t} \right|\right|\left|\left|\mathbf{h}_{b,r,t}\right|\right|}$ \hskip0.05in for $t = T_\text{wash}+T_\text{train}+1,\ldots,T$
			
			\EndFor
			
			\State $S_{b,t} \leftarrow {\displaystyle \frac{1}{\mathscr{R}}\sum_{r=1}^\mathscr{R} s_{b,r, t}}$ \hskip0.05in for $t = T_\text{wash}+T_\text{train}+1,\ldots,T$
			
			\For{$t$ in $(T_\text{wash} + T_\text{train} + 1):(T-1)$}
			
			\State $\hat{\mathcal{F}}^b_{(T_\text{wash}+T_\text{train}+1):t}(s) \leftarrow {\displaystyle \frac{1}{t - T_\text{wash} - T_\text{train}}\sum_{i=T_\text{wash} + T_\text{train} + 1}^{t}} \mathbf{1}\left\{S_{b,i} \leq s\right\}$
			
			\State $\hat{\mathcal{F}}^b_{(t+1):T}(s) \leftarrow {\displaystyle \frac{1}{T - t}\sum_{i=t + 1}^{T}} \mathbf{1}\left\{S_{b,i} \leq s\right\}$
			
			\State $K_{b,t} \leftarrow \frac{\left(t-T_\text{wash}-T_\text{train}\right)\left(T - t\right)}{q\left(t\right)\left(T-T_\text{wash}-T_\text{train}\right)^{3/2}} \; {\displaystyle \sup_s} \left|\hat{\mathcal{F}}^b_{(T_\text{wash}+T_\text{train}+1):t}(s) - \hat{\mathcal{F}}^b_{(t+1):T}(s)\right|$
			
			\EndFor
			
			\State $K_b \leftarrow {\displaystyle \max_{t}} K_{b,t}$
			
			\EndFor
			
			\State $p \leftarrow {\displaystyle \frac{1}{B} \sum_{b=1}^B}\; \mathbf{1}\left\{K_b > K \right\}$
			
			\Return $p$, all $K_b$
		\end{algorithmic}
	\end{algorithm}
	
	\newpage
	\FloatBarrier
	\subsection{Procedures Composing Main Algorithms}
	\floatname{algorithm}{Procedure}
	\begin{algorithm}
		\caption{ESN Featurization: I. Scaling}
		\label{procedure1}
		\vskip0in
		\textbf{Inputs}: time series data $\mathbf{y}_t \in \mathbb{R}^d$; training window length $T_\text{train}$
		\vskip0in
		\textbf{Outputs}: ESN scaling parameters
		\vskip0in
		\textbf{Default Parameters}: ESN spectral radius $\rho \leftarrow 0.8$; grid of possible $\mathbf{W^i}_r$, $\mathbf{b}_r$ scalings $G \leftarrow \left\{c_\text{input} \leftarrow \{0.2, 0.6, 1.0, 1.4\}, c_\text{bias} \leftarrow \{0.1, 0.3, 0.5\}\right\}$; test reservoir size $N \leftarrow 10 d$; number of initializations $\mathscr{R} \leftarrow 10$; approximate washout length $T_\text{wash}^* \leftarrow 50$;  output regularization parameter $\lambda \leftarrow 10^{-4}$
		\vskip0in
		\begin{algorithmic}[1]
			\For{each grid scaling combination of $c_\text{input}$ and $c_\text{bias}$ in $G$}
			
			\For{$r$ in $1:\mathscr{R}$}
			
			\State initialize $\mathbf{W^i}_r,\mathbf{b}_r,\mathbf{W^h}_r$ where each element $\mathcal{N}(0,1)$, and $\mathbf{W^h}_r$ is sparse
			
			\State $\mathbf{W^i}_r\leftarrow c_\text{input}\mathbf{W^i}_r$;\hskip0.1in $\mathbf{b}_r\leftarrow c_\text{bias}\mathbf{b}_r$
			
			\State $\mathbf{W^h}_r\leftarrow \rho\mathbf{W^h}_r/\max\left\{\mathbf{v}'\mathbf{W^h}_r\mathbf{v}: ||\mathbf{v}||=1\right\}$
			
			\State $\mathbf{h}_{r,t} \leftarrow \tanh\left(\mathbf{W^h}_r \mathbf{h}_{r,t-1} + \mathbf{W^i}_r \mathbf{y}_{t} + \mathbf{b}_r\right)$ \hskip0.05in for $t=1, \ldots, T_\text{wash}^*+T_\text{train}$
			
			\State $\mathbf{W^o}_r \leftarrow \left(\mathbf{H}_r\mathbf{H}_r' + \lambda\mathbf{I}\right)^{-1}\mathbf{H}_r\mathbf{Y}$ \hskip0.05in where $\mathbf{H}_r = \left[\mathbf{h}_{r,T_\text{wash}^*+1} \;\cdots\;\mathbf{h}_{r,T_\text{wash}^*+T_\text{train}}\right]$
			
			\EndFor
			
			\State $\text{NRMSE} \leftarrow {\displaystyle \frac{1}{\mathscr{R}} \sum_{j=1}^\mathscr{R}} \sqrt{\frac{\left(\mathbf{Y} -\mathbf{W^o}_r \mathbf{H}_r\right)^2}{\frac{1}{2}\text{Var}\left(\mathbf{Y}\right) + \frac{1}{2}\text{Var}\left(\mathbf{W^o}_r \mathbf{H}_r\right)}}$ 
			
			\EndFor
			
			\Return $\text{ESN scaling}: \left\{c_\text{input}, c_\text{bias}, \rho\right\} \leftarrow {\displaystyle \arg \min_{G}} \left\{\text{NRMSE}\right\}$
		\end{algorithmic}
	\end{algorithm}
	
	\floatname{algorithm}{Procedure}
	\begin{algorithm}
		\caption{ESN Featurization: II. Washout Length}
		\label{procedure2}
		\vskip0in
		\textbf{Inputs}: time series data $\mathbf{y}_t \in \mathbb{R}^d$; training window length $T_\text{train}$; $\text{ESN scaling}: \left\{c_\text{input}, c_\text{bias}, \rho\right\}$ from Procedure \ref{procedure1}; reservoir size $N$
		\vskip0in
		\textbf{Outputs}: washout length $T_\text{wash}$
		\vskip0in
		\textbf{Default Parameters}: initial reservoir states $\mathbf{h}_{r,0,0} \leftarrow 0$, $\mathbf{h}_{r,0,1} \leftarrow 1$; washout tolerance $\varepsilon_\text{wash} \leftarrow 10^{-6}$; initial time state $t\leftarrow 0$; number of initializations $\mathscr{R} \leftarrow 10$
		\vskip0in
		\begin{algorithmic}[1]
			\State $\delta_{01} \leftarrow \left|\mathbf{h}_{r,0,0}-\mathbf{h}_{r,0,1}\right|$
			\While{$\delta_{01}>\varepsilon_\text{wash}$}
			
			\For{$r$ in $1:\mathscr{R}$}
			
			\State initialize $\mathbf{W^i}_r,\mathbf{b}_r,\mathbf{W^h}_r$ where each element $\mathcal{N}(0,1)$, and $\mathbf{W^h}_r$ is sparse
			
			\State $\mathbf{W^i}_r\leftarrow c_\text{input}\mathbf{W^i}_j$;\hskip0.1in $\mathbf{b}_r\leftarrow c_\text{bias}\mathbf{b}_r$ 
			
			\State $\mathbf{W^h}_r\leftarrow \rho\mathbf{W^h}_r/\max\left\{\mathbf{v}'\mathbf{W^h}_r\mathbf{v}: ||\mathbf{v}||=1\right\}$
			
			\State $\mathbf{h}_{r,t,0} \leftarrow \tanh\left(\mathbf{W^h}_r \mathbf{h}_{r,t-1,0} + \mathbf{W^i}_r \mathbf{y}_{t} + \mathbf{b}_r\right)$
			
			\State $\mathbf{h}_{r,t,1} \leftarrow \tanh\left(\mathbf{W^h}_r \mathbf{h}_{r,t-1,1} + \mathbf{W^i}_r \mathbf{y}_{t} + \mathbf{b}_r\right)$
			
			\EndFor
			
			\State $\delta_{01} \leftarrow \max_j\left|\mathbf{h}_{r,t,0}-\mathbf{h}_{r,t,1}\right|$
			
			\If{$\delta_{01}>\varepsilon_\text{wash}$}
			
			\State $t\leftarrow t+1$
			
			\Else
			
			\State $T_\text{wash}\leftarrow t$
			
			\EndIf
			
			\EndWhile
			
			\Return $T_\text{wash}$, $\mathbf{W^i}_r,\mathbf{b}_r,\mathbf{W^h}_r$
			
		\end{algorithmic}
	\end{algorithm}
	
	\floatname{algorithm}{Procedure}
	\begin{algorithm}
		\caption{ESN Featurization: III. Parameter Computation}
		\label{procedure3}
		\vskip0in
		\textbf{Inputs}: time series data $\mathbf{y}_t \in \mathbb{R}^d$; training window length $T_\text{train}$; ESN scaling ($c_\text{input}$, $c_\text{bias}$, and $\rho\leftarrow0.8$) from Procedure \ref{procedure1}; $T_\text{wash}$ from specified value or Procedure \ref{procedure2}
		\vskip0in
		\textbf{Outputs}: ESN reservoir size $N$; aperture $\alpha$
		\vskip0in
		\textbf{Default Parameters}: training error tolerance $\varepsilon_\text{train} \leftarrow 0.04$; initial reservoir size $N \leftarrow 10 d$; initial aperture value $\alpha \leftarrow N$; number of initializations $\mathscr{R} \leftarrow 10$; output regularization parameter $\lambda \leftarrow 10^{-4}$
		\vskip0in
		\begin{algorithmic}[1]
			\While{NRMSE $> \varepsilon_\text{train}$}
			
			\State perform Procedure \ref{procedure2} to obtain $T_\text{wash}$, $\mathbf{W^i}_r,\mathbf{b}_r,\mathbf{W^h}_r$
			
			\For{$r$ in $1:\mathscr{R}$}
			
			\State $\mathbf{h}_{r,t} \leftarrow \tanh\left(\mathbf{W^h}_r \mathbf{h}_{r,t-1} + \mathbf{W^i}_r \mathbf{y}_{t} + \mathbf{b}_r\right)$ \hskip0.05in for $t=1, \ldots, T_\text{wash}+T_\text{train}$ 
			
			\State $\mathbf{C}_r \leftarrow \frac{1}{T_\text{train}}\mathbf{H}_r\mathbf{H}_r'\left(\frac{1}{T_\text{train}}\mathbf{H}_r\mathbf{H}_r' + \alpha^{-2}\mathbf{I}\right)^{-1}$ 
			\Statex \hspace*{5cm}where $\mathbf{H}_r = \left[\mathbf{h}_{r,T_\text{wash}+1} \;\cdots\;\mathbf{h}_{r,T_\text{wash}+T_\text{train}}\right]$
			
			\State $\mathbf{h}_{r,t} \leftarrow \tanh\left(\mathbf{W^h}_r \tilde{\mathbf{h}}_{r,t-1} + \mathbf{W^i}_r \mathbf{y}_{t} + \mathbf{b}_r\right)$
			\State $\tilde{\mathbf{h}}_{r,t} \leftarrow \mathbf{C}_r\mathbf{h}_{r,t}$ for $t=T_\text{wash} + 1, \ldots, T_\text{wash}+T_\text{train}$ 
			
			\State $\mathbf{W^o}_r \leftarrow \left(\tilde{\mathbf{H}}_r\tilde{\mathbf{H}}_r' + \lambda\mathbf{I}\right)^{-1}\tilde{\mathbf{H}}_r\mathbf{Y}$ \hskip0.05in where $\tilde{\mathbf{H}}_r = \left[\tilde{\mathbf{h}}_{r,T_\text{wash}+1} \;\cdots\;\tilde{\mathbf{h}}_{r,T_\text{wash}+T_\text{train}}\right]$
			
			\EndFor
			
			\State $\text{NRMSE} \leftarrow {\displaystyle \frac{1}{\mathscr{R}} \sum_{j=1}^\mathscr{R}} \sqrt{\frac{\left(\mathbf{Y} -\mathbf{W^o}_r \tilde{\mathbf{H}}_r\right)^2}{\frac{1}{2}\text{Var}\left(\mathbf{Y}\right) + \frac{1}{2}\text{Var}\left(\mathbf{W^o}_r \tilde{\mathbf{H}}_r\right)}}$ 
			
			\If{NRMSE $>\varepsilon_\text{train}$ and $\alpha \leq 100N$}
			
			\State $\alpha \leftarrow \sqrt{10}\alpha$
			
			\Else
			
			\State $N \leftarrow dN$; $\alpha \leftarrow N$
			
			\EndIf
			
			\EndWhile
			
			\Return $N$, $\alpha, T_\text{wash}$
		\end{algorithmic}
	\end{algorithm}
	
	\floatname{algorithm}{Procedure}
	\begin{algorithm}
		\caption{Generating Bootstrapped Time Series}
		\label{procedure4}
		\vskip0in
		\textbf{Inputs}: time series data $\mathbf{y}_t \in \mathbb{R}^d$; training length $T_\text{train}$; $T_\text{wash}$ from Algorithm \ref{algorithm1}
		\vskip0in
		\textbf{Outputs}: bootstrapped time series $\mathbf{y}_{b,t}$
		\vskip0in
		\textbf{Default Parameters}: number of bootstraps $B \leftarrow 240$; pilot block length $\ell \leftarrow T_\text{wash}$
		\vskip0in
		\begin{algorithmic}[1]
			\State Perform \citet{hall95} algorithm with $\ell$, five proposed block lengths equally spaced between $\lceil T^{1/5} \rceil$ and $\lceil T^{1/2} \rceil$, and 40 bootstrapped series each to obtain MBB block length parameter $L$.
			
			\For{$b$ in $1:B$}
			
			\For{$i$ in $1:\left\lceil \left(T - T_\text{wash} - T_\text{train}\right) / L \right\rceil$}
			
			\State $\beta_{b,i} \leftarrow \beta \sim \text{Uniform}\left\{T_\text{wash}+T_\text{train}+1, \; T\right\}$
			
			\State $\mathbf{b}_i \leftarrow \mathbf{y}_{\beta_{b,i}:\left(\beta_{b,i} + L - 1\right)}$
			
			\EndFor
			
			\State $\mathbf{y}_t^b \leftarrow \left[ \mathbf{y}'_{1:(T_\text{wash}+T_\text{train})} \; \mathbf{b}'_1 \;\cdots \;\mathbf{b}'_{\left\lceil \left(T - T_\text{wash} - T_\text{train}\right) / L \right\rceil}\right]'_{1:T}$
			
			\EndFor
			
			\Return all $\mathbf{y}_{b,t}$
		\end{algorithmic}
	\end{algorithm}
	
	\FloatBarrier
	\pagebreak
	\FloatBarrier
	\section*{Proofs}
	Proof of Theorem \ref{mainthm} follows \citet{csorgHo97b}, Theorem 2.6.1 with the relaxation of the i.i.d. sequence to a stationary, $\mathcal{S}$-mixing sequence. The proof consists of two major steps.  First, we show the statistic converges to a sequence of Gaussian processes (Lemma \ref{helplemma}). Because of the relaxation from an i.i.d. sequence to an $\mathcal{S}$-mixing sequence, the rates of convergence shown in \citet{csorgHo97b}, Lemma 2.6.1 are adjusted. Next we show that the sequence of Gaussian processes, in turn, converges in distribution to the desired limiting process (Lemma \ref{helplemma2}).
	\par
Define the quantity $K_T(s,t)$ in Equation \ref{KTst} on the domain $1 \leq t \leq T-1$ as a scaled difference between the two empirical CDFs. This is a common form akin to that used in \citet{csorgHo97b}, Chapter 2 with a modification of the denominator.
	\begin{align}
		K_T(s,t) &= \left[\frac{t(T-t)}{T^{2}} \right]\left(\hat{\mathcal{F}}_{1:t}(s) - \hat{\mathcal{F}}_{(t+1):T}(s)\right) \label{KTst}
	\end{align}
	\par
We scale the time domain to $\delta \in[0,1]$ such that $\delta = t / T$, and define $\mathcal{F}_{1:\delta T}(s)=\mathcal{F}_1(s)$ and $\mathcal{F}_{\delta T:T}(s)=\mathcal{F}_T(s)$ as the distribution functions on the intervals $(0,\delta]$ and $[\delta, 1)$, respectively, with $\hat{\mathcal{F}}_{1:\delta T}(s)$ and $\hat{\mathcal{F}}_{\delta T:T}(s)$ as their corresponding empirical estimates. Equation \ref{KTst} can be adjusted to the analogous form,
	\begin{align}
		K_T(s,\delta) &= \left[\frac{\delta T(T-\delta T)}{T^{2}} \right]\left(\hat{\mathcal{F}}_{1:\delta T}(s) - \hat{\mathcal{F}}_{\delta T:T}(s)\right). \label{KTsdelta}
	\end{align}
	\par
We now state and prove our Lemma \ref{helplemma}. The proof requires Theorem 2 from \citet{berkes09}, which we restate at the end of this section, without proof, for convenience.
	\begin{lemma}\label{helplemma}
		Let $S_t$ be a stationary sequence such that $\mathcal{F}(s) = P(S_1 \leq s)$ is Lipschitz continuous of order $C > 0$. Assume $S_t$ is $\mathcal{S}$-mixing and that condition (1) of Definition \ref{def:smix} holds with $\gamma_m=m^{-AC}$, $\delta_m = m^{-A}$ for some $A>4$. Then under the null hypothesis with $q(\delta)$ a positive function on $(0,1)$ that increases in a neighborhood of zero and decreases in a neighborhood of one,
		\begin{align}
			\sup_{1/T \leq \delta\leq (T-1)/T} \sup_{s\in[0,1]} \left|\sqrt{T}\;K_T(s,\delta) - \mathcal{K}_T(s,\delta)\right| /  q(\delta) = o(1),
			\label{helplemmaresult}
		\end{align}
		where $K_T(s,\delta)$ is defined in Equation \ref{KTsdelta}, $\{\mathcal{K}_T(s,\delta),\;0\leq \delta \leq 1\}$ is a sequence of Gaussian processes with
		\begin{align}\label{helplemmadetails}
			&\mathbb{E}\left[\mathcal{K}_T(s,\delta)\right] = 0,\nonumber\\
			&\mathbb{E}\left[\mathcal{K}_T(s,\delta)\;\mathcal{K}_T(s',\delta')\right] = (\delta \wedge \delta')\; \Gamma(s,s'),\nonumber\\[8pt]
			\text{and} \hskip0.2in&\Gamma(s,s') = \sum_{-\infty < t < \infty} \mathbb{E}\left[S_1(s) S_t(s')\right],
		\end{align}
		provided $I_{0,1}(q,c)<\infty$ for all $c>0$, where
		\begin{align}\label{condition}
			I_{0,1}(q,c) &= \int_0^1 \frac{1}{\delta(1-\delta)}\exp\left\{-\frac{cq^2(\delta)}{\delta(1-\delta)}\right\}d\delta.
		\end{align}
	\end{lemma}
	\begin{proof}[Proof of Lemma \ref{helplemma}]
		With $\mathcal{F}(s)$ denoting the true distribution function of all $S_t$ under the null, we can write $K_T(s,t)$ from Equation \ref{KTst} as
		\begin{align}
			\sqrt{T}\;K_T(s,t) = \begin{cases}
				&\frac{1}{\sqrt{T}} \left(\sum_{i=1}^t \left(\mathbf{1}\{S_i \leq s\} - \mathcal{F}(s)\right) - \frac{t}{T}\sum_{i=1}^T\left(\mathbf{1}\{S_i \leq s\} - \mathcal{F}(s)\right)\right)\\[12pt]
				&\hskip3.2in\text{for $1 \leq t \leq T/2$,}\\[16pt]
				&\frac{1}{\sqrt{T}}\left(\frac{T-t}{T}\sum_{i=1}^T \left(\mathbf{1}\{S_i \leq s\} - \mathcal{F}(s)\right) - \sum_{i=t+1}^T\left(\mathbf{1}\{S_i \leq s\} - \mathcal{F}(s)\right)\right)\\[12pt]
				&\hskip3.2in\text{for $T/2 \leq t \leq T-1 $}.
			\end{cases}
			\label{partialsums}
		\end{align}
		Replacing the strong approximation of empirical processes used in \citet{csorgHo97b}, Lemma 2.6.1 we  use the the $\mathcal{S}$-mixing conditions and Theorem 2 from \citet{berkes09}. Define two Gaussian processes $\{\mathcal{K}_{1}(s,t),\;1 \leq t \leq T/2\}$ and $\{\mathcal{K}_{2}(s,t),\; T/2 \leq t \leq T$\} such that
		\begin{align}
			\sup_{1\leq t \leq T/2} \sup_{s\in[0,1]} \left| \sum_{i=1}^t \left(\mathbf{1}\{S_i\leq s\} - \mathcal{F}(s)\right) - \mathcal{K}_{1}(s,t)\right|  = o\left(\left(\frac{T}{2}\right)^{1/2} \left(\log \frac{T}{2}\right)^{-\alpha}\right),
			\label{ps1}
		\end{align}
		and 
		\begin{align}
			\sup_{T/2\leq t \leq T} \sup_{s\in[0,1]} \left| \sum_{i=t+1}^T \left(\mathbf{1}\{S_i\leq s\} - \mathcal{F}(s)\right) - \mathcal{K}_{2}(s,t)\right|  = o\left(\left(\frac{T}{2}\right)^{1/2} \left(\log \frac{T}{2}\right)^{-\alpha}\right),
			\label{ps2}
		\end{align}
		for some $\alpha>0$ where the two processes have identical expected value and covariance functions.
		\begin{align}
			&\mathbb{E}\left[\mathcal{K}_{1}(s,t)\right] = \mathbb{E}\left[\mathcal{K}_{2}(s,t)\right]= 0\nonumber\\[4pt] 
			&\mathbb{E}\left[\mathcal{K}_{1}(s,t)\;\mathcal{K}_{1}(s',t')\right] = \mathbb{E}\left[\mathcal{K}_{2}(s,t)\;\mathcal{K}_{2}(s',t')\right] = 
			(t\wedge t')\;\Gamma\left(s,s'\right)\nonumber\\[4pt]
			& \Gamma(s,s') = \sum_{-\infty < t < \infty} \mathbb{E}\left[S_1(s) S_t(s')\right]\label{details}
		\end{align}
		From the strong approximation in Equations \ref{ps1} and \ref{ps2}, we define a sequence of Gaussian processes $\mathcal{K}_T(s,
		\delta)$ based on the partial sums in Equation \ref{partialsums} using the scaled time domain $\delta\in[0,1]$.
		\begin{align}\label{processeq}
			\mathcal{K}_T(s,\delta) =\begin{cases}
				\frac{1}{\sqrt{T}} \left(\mathcal{K}_1(s,\delta T) - \delta \left[\mathcal{K}_1(s,T/2) + \mathcal{K}_2(s,T/2)\right]\right) &\text{for $0\leq \delta \leq 1/2$,}\\[8pt]
				\frac{1}{\sqrt{T}} \left(-\mathcal{K}_2(s,\delta T) + (1-\delta) \left[\mathcal{K}_1(s,T/2) + \mathcal{K}_2(s,T/2)\right]\right) &\text{for $1/2\leq \delta \leq 1$}
			\end{cases}
		\end{align}
		Assembling Equations \ref{partialsums}, \ref{ps1}, and \ref{ps2}, we immediately obtain the result stated in Equation \ref{helplemmaresult}.
	\end{proof}
	\begin{lemma}\label{helplemma2}
		Let  $\{\mathcal{K}_T(s,\delta),\;0\leq \delta \leq 1\}$ be a sequence of Gaussian processes defined in Equation \ref{helplemmadetails} and $\{\mathcal{K}(s,\delta),\;0\leq \delta \leq 1\}$ be a Gaussian process defined in Equation \ref{mainthmdetails}. Under the null hypothesis with $q(\delta)$ a positive function on $(0,1)$ that increases in a neighborhood of zero and decreases in a neighborhood of one,
		\begin{align}
			\sup_{\delta\in[0,1]} \sup_{s\in[0,1]} \left|\mathcal{K}_T(s,\delta)\right| /  q(\delta) \xrightarrow{D} \sup_{\delta\in[0,1]} \sup_{s\in[0,1]} \left|\mathcal{K}(s,\delta)\right| /  q(\delta),
			\label{helplemma2result}
		\end{align}
		provided the result of Lemma \ref{helplemma} holds for $K_T(s,\delta)$ defined in Equation \ref{KTsdelta}, and $I_{0,1}(q,c)<\infty$ for all $c>0$, where $I_{0,1}(q,c)$ is defined in Equation \ref{condition}.
	\end{lemma}
	\begin{proof}[Proof of Lemma \ref{helplemma2}]
		We continue as in Theorem 2.6.1 from \citet{csorgHo97b}. From the definition of $K_T(s,\delta)$ in Equation \ref{KTsdelta} at the extreme ends of the domain,
		\begin{align}\label{lowerpiece}
			\sup_{0<\delta<1/T} \sup_{s\in[0,1]} \left|\sqrt{T}\;K_T(s,\delta)\right| /  q(\delta) = 0,
		\end{align}
		and
		\begin{align}\label{higherpiece}
			\sup_{(T-1)/T<\delta<1} \sup_{s\in[0,1]} \left|\sqrt{T}\;K_T(s,\delta)\right| /  q(\delta) = 0.
		\end{align}
		From the result of Lemma \ref{helplemma} in Equation \ref{helplemmaresult} with the condition $I_{0,1}(q,c)<\infty$ for all $c>0$, we can obtain Equations \ref{lowerpiece2} and \ref{higherpiece2}.
		\begin{align}
			\sup_{0<\delta<1/T} \sup_{s\in[0,1]} \left|\mathcal{K}_T(s,\delta)\right| /  q(\delta) = o(1)\label{lowerpiece2}\\
			\sup_{(T-1)/T<\delta<1} \sup_{s\in[0,1]} \left|\mathcal{K}_T(s,\delta)\right| /  q(\delta) = o(1)\label{higherpiece2}
		\end{align}
		Examining the covariance structure of $\mathcal{K}_T(s,\delta)$ will verify that
		\begin{align}\label{distconverge}
			\left\{\mathcal{K}_T(s,\delta),\;0\leq \delta \leq 1\right\} \xrightarrow{D} \left\{\mathcal{K}(s,\delta),\;0\leq \delta \leq 1\right\},
		\end{align}
		and via Kolmogorov's zero-one law and Theorem A.7.3 from \citet{csorgHo97b}, 
		\begin{align}
			\lim_{\delta\downarrow0}\sup_{s\in[0,1]}\left|\mathcal{K}(s,\delta)\right|/q(\delta)&=0 \hskip0.1in\text{a.s.}\label{lowerp2}\\
			\text{and}\hskip0.1in \lim_{\delta\uparrow1}\sup_{s\in[0,1]}\left|\mathcal{K}(s,\delta)\right|/q(\delta)&=0 \hskip0.1in\text{a.s.},\label{higherp2}
		\end{align}
		if and only if $I_{0,1}(q,c)<\infty$ for all $c>0$. Thus, we obtain the Lemma \ref{helplemma2} result stated in Equation \ref{helplemma2result}.
	\end{proof}
	\par
We now use Lemmas \ref{helplemma} and \ref{helplemma2} to prove Theorem \ref{mainthm}. 
	\begin{proof}[Proof of Theorem \ref{mainthm}]
		We combine the result from Lemma \ref{helplemma} in Equation \ref{helplemmaresult} with that of Lemma \ref{helplemma2} in Equation \ref{helplemma2result} and can write
		\begin{align}
			\max_{1 \leq t < T}\sup_{s\in[0,1]}\sqrt{T}\, \frac{K_T(s,t)} {q\left(\frac{t}{T} \right)} \xrightarrow{D} \sup_{\delta \in [0,1]}  \sup_{s\in[0,1]} \left| \mathcal{K}(s, \delta)\right| / q(\delta), 
		\end{align}
		as shown in in Equation \ref{mainthmresult} provided the necessary condition, $I_{0,1}(q,c)<\infty$ for all $c>0$, is met with $I_{0,1}(q,c)$ defined in Equation \ref{condition}. 
		We expand $q(\delta)$, defined in in Theorem \ref{mainthm}, as a piecewise function on $\delta\in[0,1]$.
		\begin{align}
			q(\delta) = \begin{cases}
				\kappa &\text{for $0 \leq \delta < \frac{1}{2} - \frac{1}{2}\sqrt{1-4\kappa^{1/\nu}}$,}\\[6pt]
				\delta^\nu(1-\delta)^\nu &\text{for $\frac{1}{2} - \frac{1}{2}\sqrt{1-4\kappa^{1/\nu}} \leq \delta \leq \frac{1}{2} + \frac{1}{2}\sqrt{1-4\kappa^{1/\nu}}$,}\\[6pt]
				\kappa &\text{for $\frac{1}{2} + \frac{1}{2}\sqrt{1-4\kappa^{1/\nu}} < \delta \leq 1$}
			\end{cases}
		\end{align}
		The integral from Equation \ref{condition} becomes
		\begin{align}
			I_{0,1}(q,c) &= \int_0^{\frac{1}{2} - \frac{1}{2}\sqrt{1-4\kappa^{1/\nu}}} \frac{1}{\delta(1-\delta)}\exp\left\{-c\kappa^2\delta^{-1}(1-\delta)^{-1}\right\}d\delta\nonumber\\[6pt]
			&\hskip0.1in+ \int_{\frac{1}{2} - \frac{1}{2}\sqrt{1-4\kappa^{1/\nu}}}^{\frac{1}{2} + \frac{1}{2}\sqrt{1-4\kappa^{1/\nu}}} \frac{1}{\delta(1-\delta)}\exp\left\{-c\delta^{2\nu-1}(1-\delta)^{2\nu-1}\right\}d\delta \nonumber\\[6pt]
			&\hskip0.1in+ \int_{\frac{1}{2} + \frac{1}{2}\sqrt{1-4\kappa^{1/\nu}}}^1 \frac{1}{\delta(1-\delta)}\exp\left\{-c\kappa^2\delta^{-1}(1-\delta)^{-1}\right\}d\delta,
		\end{align}
		and for any $c>0$ the boundary terms are finite with $\kappa>0$.
		\begin{align}
			\int_0^{\frac{1}{2} - \frac{1}{2}\sqrt{1-4\kappa^{1/\nu}}} \frac{1}{\delta(1-\delta)}\exp\left\{-c\kappa^2\delta^{-1}(1-\delta)^{-1}\right\}d\delta &< \infty\\[6pt]
			\int_{\frac{1}{2} + \frac{1}{2}\sqrt{1-4\kappa^{1/\nu}}}^1 \frac{1}{\delta(1-\delta)}\exp\left\{-c\kappa^2\delta^{-1}(1-\delta)^{-1}\right\}d\delta &<\infty
		\end{align}
		For the middle term,
		\begin{align}
			\int_{\frac{1}{2} - \frac{1}{2}\sqrt{1-4\kappa^{1/\nu}}}^{\frac{1}{2} + \frac{1}{2}\sqrt{1-4\kappa^{1/\nu}}} \frac{1}{\delta(1-\delta)}\exp\left\{-c\delta^{2\nu-1}(1-\delta)^{2\nu-1}\right\}d\delta &< \infty
			\label{middleterm}
		\end{align}
		provided $\nu<1/2$ if $\kappa \rightarrow 0$. For $\kappa>0$, the range of values satisfying Equation \ref{middleterm} includes $\nu=1/2$. Thus, from the specification of $q(\delta)$ in Theorem \ref{mainthm} with $\nu=1/2$ and $\kappa>0$, $I_{0,1}(q,c)<\infty$ for all $c>0$.
	\end{proof}
	\par 
	Proof of Theorem \ref{consistent} follows the outline of Theorem 2.1 from \citet{newey94} for consistency of extremum estimators. To meet the first three of the four conditions, we show the statistic for selection of a change point is uniquely maximized at the true change point $\tau$, the set used for estimation is compact and bounded away from the endpoints, and the statistic is continuous. To show the estimate of the statistic converges uniformly in probability to the true values, we employ the almost sure convergence of the empirical CDF of a stationary, ergodic sequence from the Glivenko-Cantelli Theorem and the definition of stochastic equicontinuity and Theorem 1 from \citet{newey91}, along with Lemmas \ref{newseq} and \ref{c4}. We include the definition of stochastic equicontinuity from \citet{newey91} and restate Theorem 2.1, without proof, from \citet{newey94} at the end of the section. We first proceed with our proof of Lemma \ref{newseq}.
	\begin{lemma}
		A sequence of functions $\hat{Q}_T(\delta)$ is stochastically equicontinuous if there exists $\alpha>0$, $\hat{A}_T=o(1)$, and $\hat{B}_T=\mathcal{O}(1)$ such that for all $\tilde{\delta},\delta\in\Delta$, $|\hat{Q}_T(\tilde{\delta})-\hat{Q}_T(\delta)|\leq \hat{A}_T + \hat{B}_T\|\tilde{\delta}-\delta\|^\alpha$.
		\label{newseq}
	\end{lemma}
	\begin{proof}[Proof of Lemma \ref{newseq}]
		We follow the strategy of the proof for Lemma 2.9 in \citet{newey94}. Pick $\varepsilon, \eta >0$. By $\hat{A}_T=o(1)$, $\mathbb{P}(|\hat{A}_T|>\frac{\varepsilon}{2})<\frac{\eta}{2}$ for $T$ large enough. Likewise, by $\hat{B}_T=\mathcal{O}(1)$, there is $M$ such that $\mathbb{P}(|\hat{B}_T|>M)<\frac{\eta}{2}$ for all $T$ large enough. Let $\Gamma_T(\varepsilon,\eta)=\hat{A}_T+\hat{B}_T\frac{\varepsilon}{2M}$ and $\mathcal{N}(\delta,\varepsilon,\eta)=\{\tilde{\delta}:\|\tilde{\delta}-\delta\|^\alpha \leq \frac{\varepsilon}{2M}\}$. Then, $\mathbb{P}(|\Gamma_T(\varepsilon,\eta)|>\varepsilon)=\mathbb{P}(|\hat{A}_T+\hat{B}_T\frac{\varepsilon}{2M}|>\varepsilon)\leq \mathbb{P}(|\hat{A}_T|+|\hat{B}_T\frac{\varepsilon}{2M}|>\varepsilon) \leq \mathbb{P}(|\hat{A}_T|>\frac{\varepsilon}{2})+\mathbb{P}(|\hat{B}_T|>M)<\eta$ and for all $\tilde{\delta},\delta\in\mathcal{N}(\delta,\varepsilon,\eta)$, $|\hat{Q}_T(\tilde{\delta})-\hat{Q}_T(\delta)|\leq \hat{A}_T+\hat{B}_T\|\tilde{\delta}-\delta\|^{\alpha}\leq \Gamma_T(\varepsilon,\eta)$.
	\end{proof}
	\par
	For the statement and proof of Lemma \ref{c4}, we scale the time domain to $\delta \in[0,1]$ such that $\delta = t / T$, with the true change point at $\delta_0 = \tau / T$, and 
	define $\mathcal{F}_{1:\delta T}(s)$ and $\mathcal{F}_{\delta T:T}(s)$ as the distribution functions on the intervals $(0,\delta]$ and $[\delta, 1)$, respectively, with $\hat{\mathcal{F}}_{1:\delta T}(s)$ and $\hat{\mathcal{F}}_{\delta T:T}(s)$ as their corresponding empirical estimates. Suppose the true change point occurs at $\delta_0\in\Delta$, and divides the sequence into two distinct pieces with distribution functions $\mathcal{F}_1(s)=\mathcal{F}_{1:\delta_0 T}(s)$ and $\mathcal{F}_T(s)=\mathcal{F}_{\delta_0 T: T}(s)$, where $\mathcal{F}_1(s_0)  \neq  \mathcal{F}_T(s_0)$ for some $s_0\in[0,1]$. Define the supremum of the difference between the two distributions $\theta>0$, where the quantity is maximized at $s_0$.
	\begin{align}
		\theta &= \sup_{s\in[0,1]}  \left|\mathcal{F}_1(s) - \mathcal{F}_T(s)\right| = \left|\mathcal{F}_1(s_0) - \mathcal{F}_T(s_0)\right|
		\label{theta}
	\end{align} 
	\vskip0in
	Write $Q_0(\delta)$ as below for $\delta \in \Delta = \left[\frac{1}{2}-\frac{1}{2}\sqrt{1-4\kappa^2}, \frac{1}{2} + \frac{1}{2}\sqrt{1-4\kappa^2}\right]$, $q(\delta) = \max\{\delta^{1/2}(1-\delta)^{1/2}, \kappa\}$, and $\kappa>0$ a small constant.
	\begin{align}
		Q_0(\delta) &= \frac{1}{q(\delta)}\left[\frac{\delta T \left(T-\delta T\right)}{T^2}\right] \sup_{s\in[0,1]} \left|\mathcal{F}_{1:\delta T}(s) - \mathcal{F}_{\delta T:T}(s)\right|
		\label{q0}
	\end{align}
	We construct $\hat{Q}_T(\delta)$ in Equation \ref{qt} comparably to $Q_0(\delta)$ while including the empirical CDFs $\hat{\mathcal{F}}_{1:\delta T}(s)$ and $\hat{\mathcal{F}}_{\delta T:T}(s)$.
	\begin{align}
		\hat{Q}_T(\delta) &= \frac{1}{q\left(\delta \right)}\left[\frac{\delta T(T-\delta T)}{T^2} \right] \sup_{s\in[0,1]}\left|\hat{\mathcal{F}}_{1:\delta T}(s) - \hat{\mathcal{F}}_{\delta T:T}(s)\right|
		\label{qt}
	\end{align}
	\begin{lemma}\label{c4}
		Suppose the sequence $S_t$, $1,\ldots,T$ divides into two stationary ergodic pieces on either side of the change point $\tau$, where $1\leq \tau < T$, and $\mathcal{F}_{1:\tau}(s_0) \neq \mathcal{F}_{(\tau+1):T}(s_0)$ for some $s_0\in[0,1]$. Then, $\hat{Q}_T(\delta)$ from Equation \ref{qt} converges uniformly in probability to $Q_0(\delta)$ from Equation \ref{q0} on the interval 
		\begin{align}
			\delta,\delta_0\in\Delta = \left[\frac{1}{2}-\frac{1}{2}\sqrt{1-4\kappa^2}, \frac{1}{2} + \frac{1}{2}\sqrt{1-4\kappa^2}\right]
		\end{align} with $q(\delta) = \max\{\delta^{1/2}(1-\delta)^{1/2}, \kappa\}$ for and some small $\kappa>0$.
		
	\end{lemma}
	\begin{proof}[Proof of Lemma \ref{c4}]
		We use Theorem 1 from \citet{newey91} and show uniform convergence in probability with pointwise convergence and stochastic equicontinuity. For any $\delta,\delta_0\in\Delta$, we can write the distributions $\mathcal{F}_{1:\delta T}(s)$ and $\mathcal{F}_{\delta T:T}(s)$ as a mixture between the distributions $\mathcal{F}_1(s)$ and $\mathcal{F}_T(s)$.
		\begin{align}
			\mathcal{F}_{1:\delta T}(s) &= \mathbf{1}\{\delta \leq \delta_0\}\mathcal{F}_1(s) + \mathbf{1}\{\delta > \delta_0\}\left[\frac{\delta_0}{\delta}\mathcal{F}_1(s) + \frac{\delta-\delta_0}{\delta}\mathcal{F}_T(s)\right]\label{d1}\\[8pt]
			\mathcal{F}_{\delta T: T}(s) &= \mathbf{1}\{\delta \leq \delta_0\}\left[\frac{\delta_0-\delta}{1-\delta}\mathcal{F}_1(s) + \frac{
				1-\delta_0}{1-\delta}\mathcal{F}_T(s)\right] + \mathbf{1}\{\delta > \delta_0\}\mathcal{F}_T(s) \label{d2} 
		\end{align}
		\vskip0in
		Taking the supremum of the difference between the distributions in Equations \ref{d1} and \ref{d2} as it appears in $Q_0(\delta)$,
		\begin{align}
			\sup_{s\in[0,1]} &\left|\mathcal{F}_{1:\delta T}(s) - \mathcal{F}_{\delta T: T}(s)\right| \nonumber\\
			&= \sup_{s\in[0,1]} \bigg|  \mathbf{1}\{\delta \leq \delta_0\}\mathcal{F}_1(s) + \mathbf{1}\{\delta > \delta_0\}\left[\frac{\delta_0}{\delta}\mathcal{F}_1(s) + \frac{\delta-\delta_0}{\delta}\mathcal{F}_T(s)\right] \nonumber\\
			& \hskip0.4in -\mathbf{1}\{\delta \leq \delta_0\}\left[\frac{\delta_0-\delta}{1-\delta}\mathcal{F}_1(s) + \frac{
				1-\delta_0}{1-\delta}\mathcal{F}_T(s)\right] + \mathbf{1}\{\delta > \delta_0\}\mathcal{F}_T(s) \bigg|\\[8pt]
			&= \sup_{s\in[0,1]} \bigg| \mathbf{1}\{\delta \leq \delta_0\}\left[\frac{1-\delta_0}{1-\delta}\mathcal{F}_1(s) - \frac{
				1-\delta_0}{1-\delta}\mathcal{F}_T(s)\right] \nonumber\\
			&\hskip0.4in -\mathbf{1}\{\delta > \delta_0\}\left[\frac{\delta_0}{\delta}\mathcal{F}_1(s) - \frac{\delta_0}{\delta}\mathcal{F}_T(s)\right]\bigg| \\[8pt]
			&= \sup_{s\in[0,1]} \bigg| \left[\mathbf{1}\{\delta \leq \delta_0\}\frac{1-\delta_0}{1-\delta} - \mathbf{1}\{\delta > \delta_0\}\frac{\delta_0}{\delta}\right]\big(\mathcal{F}_1(s) -\mathcal{F}_T(s)\big)\bigg| \\[8pt]
			&= \left|\mathbf{1}\{\delta \leq \delta_0\}\frac{1-\delta_0}{1-\delta} - \mathbf{1}\{\delta > \delta_0\}\frac{\delta_0}{\delta}\right|\sup_{s\in[0,1]} \left|\mathcal{F}_1(s) - \mathcal{F}_T(s)\right| \\[8pt]
			&= \left|\mathbf{1}\{\delta \leq \delta_0\}\frac{1-\delta_0}{1-\delta} - \mathbf{1}\{\delta > \delta_0\}\frac{\delta_0}{\delta}\right|\theta.  \label{result1}
		\end{align}
		Considering the leading term in Equation \ref{q0},
		\begin{align}
			Q_0(\delta) =
			\begin{cases}
				{\displaystyle \theta \frac{\delta^{1/2}(1-\delta_0)}{(1-\delta)^{1/2}}} &\text{for $\delta,\delta_0 \in\Delta$ and $\delta \leq \delta_0$,}\\[14pt]
				{\displaystyle \theta \frac{\delta_0(1-\delta)^{1/2}}{\delta^{1/2}}} &\text{for $\delta,\delta_0 \in\Delta$ and $\delta > \delta_0$.}
			\end{cases}
			\label{q0pieces}
		\end{align}
		\vskip0in
		We expand the supremum term of $\hat{Q}_T(\delta)$ in a similar fashion depending on the relationship between $\delta$ and $\delta_0$. For $\delta\leq \delta_0$,
		\begin{align}
			\sup_{s\in[0,1]}\left|\hat{\mathcal{F}}_{1:\delta T}(s) - \hat{\mathcal{F}}_{\delta T:T}(s)\right| &= \sup_{s\in[0,1]}\left|\hat{\mathcal{F}}_{1:\delta T}(s) - \frac{\delta_0-\delta}{1-\delta}\hat{\mathcal{F}}_{\delta T:\delta_0T}(s) -  \frac{1-\delta_0}{1-\delta}\hat{\mathcal{F}}_{\delta_0 T:T}(s) \right|\nonumber\\[8pt]
			&= \sup_{s\in[0,1]}\bigg|\left(\hat{\mathcal{F}}_{1:\delta T}(s)-\mathcal{F}_1(s)\right) + \frac{\delta_0-\delta}{1-\delta}\left(\mathcal{F}_1(s)-\hat{\mathcal{F}}_{\delta T:\delta_0T}(s)\right) \nonumber\\[6pt]
			&\hskip0.2in+ \frac{1-\delta_0}{1-\delta}\left(\mathcal{F}_T(s)-\hat{\mathcal{F}}_{\delta_0 T:T}(s)\right) \bigg|\\[8pt]
			&\leq \sup_{s\in[0,1]}\left|\hat{\mathcal{F}}_{1:\delta T}(s)-\mathcal{F}_1(s)\right| + \frac{\delta_0-\delta}{1-\delta}\sup_{s\in[0,1]}\left|\hat{\mathcal{F}}_{\delta T:\delta_0 T}(s)-\mathcal{F}_1(s)\right|\nonumber\\[6pt]
			&\hskip0.2in + \frac{1-\delta_0}{1-\delta}\sup_{s\in[0,1]}\left|\hat{\mathcal{F}}_{\delta_0 T: T}(s)-\mathcal{F}_T(s)\right|\nonumber\\[6pt]
			&\hskip0.2in + \frac{1-\delta_0}{1-\delta}\sup_{s\in[0,1]}\left|\mathcal{F}_1(s)-\mathcal{F}_T(s)\right|.\label{qtn1}
		\end{align}
		And for $\delta > \delta_0$,
		\begin{align}
			\sup_{s\in[0,1]}\left|\hat{\mathcal{F}}_{1:\delta T}(s) - \hat{\mathcal{F}}_{\delta T:T}(s)\right| &\leq \frac{\delta_0}{\delta}\sup_{s\in[0,1]}\left|\hat{\mathcal{F}}_{1:\delta_0 T}(s)-\mathcal{F}_1(s)\right| + \frac{\delta-\delta_0}{\delta}\sup_{s\in[0,1]}\left|\hat{\mathcal{F}}_{\delta_0 T:\delta T}(s)-\mathcal{F}_T(s)\right|\nonumber\\[6pt]
			&\hskip0.2in + \sup_{s\in[0,1]}\left|\hat{\mathcal{F}}_{\delta T: T}(s)-\mathcal{F}_T(s)\right| + \frac{\delta_0}{\delta}\sup_{s\in[0,1]}\left|\mathcal{F}_1(s)-\mathcal{F}_T(s)\right|.\label{qtn2}
		\end{align}
		Using an extension of the Glivenko-Cantelli theorem to stationary and ergodic sequences, the supremum term almost surely converges to a scaled difference in the true distribution functions as the number of time points in each section gets large \citep{tucker59, yu93, dehling02}.
		\begin{align}
			\mathbb{P}\left[\left(\sup_{s\in[0,1]}\left|\hat{\mathcal{F}}_{1:\delta T}(s)  - \hat{\mathcal{F}}_{\delta T:T}(s)\right|-\frac{1-\delta_0}{1-\delta}\theta\right)\rightarrow 0\right]  &= 1
			\hskip0.1in\text{for $\delta \leq \delta_0$,} \label{n3}\\[8pt]
			\text{and} \hskip0.1in\mathbb{P}\left[\left(\sup_{s\in[0,1]}\left|\hat{\mathcal{F}}_{1:\delta T}(s)  - \hat{\mathcal{F}}_{\delta T:T}(s)\right|-\frac{\delta_0}{\delta}\theta\right)\rightarrow 0\right] &= 1\hskip0.1in\text{for $\delta > \delta_0$.}\label{n4}
		\end{align}
		Including the leading coefficient,
		\begin{align}
			\hat{Q}_T(\delta) \xrightarrow{a.s.} 
			\begin{cases}
				{\displaystyle \theta\frac{\delta^{1/2}(1-\delta_0)}{(1-\delta)^{1/2}}} &\text{for $\delta\leq \delta_0$,}\\[14pt]
				{\displaystyle\theta\frac{\delta_0(1-\delta)^{1/2}}{\delta^{1/2}}} &\text{for $\delta > \delta_0$,}
			\end{cases}
		\end{align}
		and we obtain $\hat{Q}_T(\delta) \rightarrow Q_0(\delta)$ with probability 1 for all $\delta,\delta_0 \in \Delta$. 
		\vskip0in
		To show $\hat{Q}_T(\delta)$ is stochastically equicontinuous, we use the the structure of Lemma \ref{newseq}. Define $\alpha=\kappa$, $\hat{B}_T = C\theta$, where $C=(2/\kappa)\sqrt{1-\kappa^2} + 1$, and $\hat{A}_T$ as below in Equations \ref{hatAT1} through \ref{hatAT6} depending on the relationship between $\tilde{\delta}$, $\delta$, and $\delta_0$ where each is contained in $\Delta$. For each case $i=1,\ldots,6$, $\hat{A}^i_T=o(1)$ is a weaker condition than the extended Glivenko-Cantelli result, and $\hat{B}_T=\mathcal{O}(1)$ is verified with $C<\infty$, $0< \theta \leq 1$ \citep{tucker59, yu93}.
		\begin{align}
			\hat{A}^1_T &= \sup_{s\in[0,1]}\left|\hat{\mathcal{F}}_{1:\tilde{\delta}T}(s) - \mathcal{F}_1(s)\right| + \frac{3}{2}\sup_{s\in[0,1]}\left|\hat{\mathcal{F}}_{\tilde{\delta}T:\delta T}(s) - \mathcal{F}_1(s)\right|\nonumber\\[6pt]
			&\hskip0.3in + \sup_{s\in[0,1]}\left|\hat{\mathcal{F}}_{\delta T:\delta_0 T}(s) - \mathcal{F}_1(s)\right| +
			\sup_{s\in[0,1]}\left|\hat{\mathcal{F}}_{\delta_0 T:T}(s) - \mathcal{F}_T(s)\right| \hskip0.1in\text{for $\tilde{\delta}<\delta\leq \delta_0$,}\label{hatAT1}\\[10pt]
			\hat{A}^2_T &= \sup_{s\in[0,1]}\left|\hat{\mathcal{F}}_{1:\delta T}(s) - \mathcal{F}_1(s)\right| + \frac{3}{2}\sup_{s\in[0,1]}\left|\hat{\mathcal{F}}_{\delta T:\tilde{\delta} T}(s) - \mathcal{F}_1(s)\right|\nonumber\\[6pt]
			&\hskip0.3in + \sup_{s\in[0,1]}\left|\hat{\mathcal{F}}_{\tilde{\delta}T:\delta_0 T}(s) - \mathcal{F}_1(s)\right| +
			\sup_{s\in[0,1]}\left|\hat{\mathcal{F}}_{\delta_0 T:T}(s) - \mathcal{F}_T(s)\right| \hskip0.1in\text{for $\delta<\tilde{\delta}\leq \delta_0$,}\label{hatAT2}\\[10pt]
			\hat{A}^3_T &= \sup_{s\in[0,1]}\left|\hat{\mathcal{F}}_{1:\tilde{\delta} T}(s) - \mathcal{F}_1(s)\right| + \frac{3}{2}\sup_{s\in[0,1]}\left|\hat{\mathcal{F}}_{\tilde{\delta} T:\delta_0 T}(s) - \mathcal{F}_1(s)\right|\nonumber\\[6pt]
			&\hskip0.3in + \frac{3}{2}\sup_{s\in[0,1]}\left|\hat{\mathcal{F}}_{\delta_0 T:\delta T}(s) - \mathcal{F}_T(s)\right| +
			\sup_{s\in[0,1]}\left|\hat{\mathcal{F}}_{\delta T:T}(s) - \mathcal{F}_T(s)\right| \hskip0.1in\text{for $\tilde{\delta}\leq \delta_0< \delta$,}\label{hatAT3}\\[10pt]
			\hat{A}^4_T &=\sup_{s\in[0,1]}\left|\hat{\mathcal{F}}_{1:\delta T}(s) - \mathcal{F}_1(s)\right| + \frac{3}{2}\sup_{s\in[0,1]}\left|\hat{\mathcal{F}}_{\delta T:\delta_0 T}(s) - \mathcal{F}_1(s)\right|\nonumber\\[6pt]
			&\hskip0.3in + \frac{3}{2}\sup_{s\in[0,1]}\left|\hat{\mathcal{F}}_{\delta_0 T:\tilde{\delta} T}(s) - \mathcal{F}_T(s)\right| +
			\sup_{s\in[0,1]}\left|\hat{\mathcal{F}}_{\tilde{\delta} T:T}(s) - \mathcal{F}_T(s)\right| \hskip0.1in\text{for $\delta\leq \delta_0< \tilde{\delta}$,}\label{hatAT4}
		\end{align}
		\begin{align}
			\hat{A}^5_T &=\sup_{s\in[0,1]}\left|\hat{\mathcal{F}}_{1:\delta_0 T}(s) - \mathcal{F}_1(s)\right| + \sup_{s\in[0,1]}\left|\hat{\mathcal{F}}_{\delta_0 T:\tilde{\delta} T}(s) - \mathcal{F}_T(s)\right|\nonumber\\[6pt]
			&\hskip0.3in + \frac{3}{2}\sup_{s\in[0,1]}\left|\hat{\mathcal{F}}_{\tilde{\delta} T:\delta T}(s) - \mathcal{F}_T(s)\right| +
			\sup_{s\in[0,1]}\left|\hat{\mathcal{F}}_{\delta T:T}(s) - \mathcal{F}_T(s)\right| \hskip0.1in\text{for $\delta_0< \tilde{\delta}< \delta$,}\label{hatAT5}\\[10pt]
			\hat{A}^6_T &=\sup_{s\in[0,1]}\left|\hat{\mathcal{F}}_{1:\delta_0 T}(s) - \mathcal{F}_1(s)\right| + \sup_{s\in[0,1]}\left|\hat{\mathcal{F}}_{\delta_0 T:\delta T}(s) - \mathcal{F}_T(s)\right|\nonumber\\[6pt]
			&\hskip0.3in + \frac{3}{2}\sup_{s\in[0,1]}\left|\hat{\mathcal{F}}_{\delta T:\tilde{\delta} T}(s) - \mathcal{F}_T(s)\right| +
			\sup_{s\in[0,1]}\left|\hat{\mathcal{F}}_{\tilde{\delta} T:T}(s) - \mathcal{F}_T(s)\right| \hskip0.1in\text{for $\delta_0< \delta< \tilde{\delta}$}\label{hatAT6}
		\end{align}
		\vskip0in
		We simplify the form of $\hat{Q}_T$ in Equation \ref{qt} and write $q(\delta)=\delta^{1/2}(1-\delta)^{1/2}$ for the domain restriction $\tilde{\delta},\delta\in\Delta$ such that
		\begin{align}
			\left|\hat{Q}_T(\tilde{\delta}) - \hat{Q}_T(\delta)\right| &= \bigg|q(\tilde{\delta}) \sup_{s\in[0,1]}\left|\hat{\mathcal{F}}_{1:\tilde{\delta} T}(s) - \hat{\mathcal{F}}_{\tilde{\delta} T:T}(s)\right| \nonumber\\[8pt]
			&\hskip0.3in- q(\delta)\sup_{s\in[0,1]}\left|\hat{\mathcal{F}}_{1:\delta T}(s) - \hat{\mathcal{F}}_{\delta T:T}(s)\right|\bigg|,\label{difference}
		\end{align}
		and then alter the expression to a convenient final form shown in Equation \ref{finaldiff}, where all absolute value terms are separated for easy manipulation.
		\begin{align}
			&\left|\hat{Q}_T(\tilde{\delta}) - \hat{Q}_T(\delta)\right| \nonumber\\
			&\hskip0.3in=\bigg|q(\tilde{\delta}) \sup_{s\in[0,1]}\left|\hat{\mathcal{F}}_{1:\tilde{\delta} T}(s) - \hat{\mathcal{F}}_{1:\delta T}(s) + \hat{\mathcal{F}}_{1:\delta T}(s) - \hat{\mathcal{F}}_{\tilde{\delta} T:T}(s) + \hat{\mathcal{F}}_{\delta T:T}(s) - \hat{\mathcal{F}}_{\delta T:T}(s)\right| \nonumber\\[8pt]
			&\hskip0.6in- q(\delta)\sup_{s\in[0,1]}\left|\hat{\mathcal{F}}_{1:\delta T}(s) - \hat{\mathcal{F}}_{\delta T:T}(s)\right|\bigg|\\[12pt]
			&\hskip0.3in\leq \bigg| q(\tilde{\delta})\sup_{s\in[0,1]}\left|\hat{\mathcal{F}}_{1:\tilde{\delta} T}(s) - \hat{\mathcal{F}}_{1:\delta T}(s)\right| + q(\tilde{\delta})\sup_{s\in[0,1]}\left|\hat{\mathcal{F}}_{\tilde{\delta} T:T}(s) - \hat{\mathcal{F}}_{\delta T:T}(s)\right|\nonumber\\[8pt]
			&\hskip0.6in + \left(q(\tilde{\delta}) - q(\delta)\right)\sup_{s\in[0,1]}\left|\hat{\mathcal{F}}_{1:\delta T}(s) - \hat{\mathcal{F}}_{\delta T:T}(s)\right|\bigg|\\[12pt]
			&\hskip0.3in\leq  q(\tilde{\delta})\sup_{s\in[0,1]}\left|\hat{\mathcal{F}}_{1:\tilde{\delta} T}(s) - \hat{\mathcal{F}}_{1:\delta T}(s)\right| + q(\tilde{\delta})\sup_{s\in[0,1]}\left|\hat{\mathcal{F}}_{\tilde{\delta} T:T}(s) - \hat{\mathcal{F}}_{\delta T:T}(s)\right|\nonumber\\[8pt]
			&\hskip0.6in + \left|q(\tilde{\delta}) - q(\delta)\right|\sup_{s\in[0,1]}\left|\hat{\mathcal{F}}_{1:\delta T}(s) - \hat{\mathcal{F}}_{\delta T:T}(s)\right|\label{finaldiff}
		\end{align}
		For each piece in Equation \ref{finaldiff}, we examine the result based on the relationship between $\tilde{\delta}$, $\delta$, and $\delta_0$. The full exposition is given below for the case where $\tilde{\delta}<\delta\leq \delta_0$ to obtain $\hat{A}^1_T$, and abbreviated versions are included for the remaining pieces $\hat{A}_T^2$ through $\hat{A}_T^6$ that follow a similar structure. 
		\vskip0in
		We proceed under condition one of six, where $\tilde{\delta}<\delta\leq \delta_0$. For the terms in Equation \ref{finaldiff}, we separate the form into distinct pieces of the empirical CDFs and introduce the true distribution functions $\mathcal{F}_1(s)$ and $\mathcal{F}_T(s)$.
		\begin{align}
			&q(\tilde{\delta})\sup_{s\in[0,1]}\left|\hat{\mathcal{F}}_{1:\tilde{\delta} T}(s) - \hat{\mathcal{F}}_{1:\delta T}(s)\right| \nonumber\\[8pt]
			&\hskip0.3in = q(\tilde{\delta})\sup_{s\in[0,1]}\left|\hat{\mathcal{F}}_{1:\tilde{\delta} T}(s) - \frac{\tilde{\delta}}{\delta}\hat{\mathcal{F}}_{1:\tilde{\delta} T}(s) - \frac{\delta-\tilde{\delta}}{\delta}\hat{\mathcal{F}}_{\tilde{\delta}T:\delta T}(s)\right|\\[8pt]
			&\hskip0.3in = q(\tilde{\delta})\sup_{s\in[0,1]}\left|\frac{\delta-\tilde{\delta}}{\delta}\hat{\mathcal{F}}_{1:\tilde{\delta} T}(s) - \frac{\delta-\tilde{\delta}}{\delta}\hat{\mathcal{F}}_{\tilde{\delta}T:\delta T}(s)\right|\\[8pt]
			&\hskip0.3in = q(\tilde{\delta})\sup_{s\in[0,1]}\left|\frac{\delta-\tilde{\delta}}{\delta}\left(\hat{\mathcal{F}}_{1:\tilde{\delta} T}(s)-\mathcal{F}_1(s)\right) - \frac{\delta-\tilde{\delta}}{\delta}\left(\hat{\mathcal{F}}_{\tilde{\delta}T:\delta T}(s)-\mathcal{F}_1(s)\right)\right|\\[8pt]
			&\hskip0.3in \leq q(\tilde{\delta})\sup_{s\in[0,1]}\left|\hat{\mathcal{F}}_{1:\tilde{\delta}T}(s) - \mathcal{F}_1(s)\right| + q(\tilde{\delta})\sup_{s\in[0,1]}\left|\hat{\mathcal{F}}_{\tilde{\delta}T:\delta T}(s)-\mathcal{F}_1(s)\right|\label{term1case1}
		\end{align}
		For the second term,
		\begin{align}
			&q(\tilde{\delta})\sup_{s\in[0,1]}\left|\hat{\mathcal{F}}_{\tilde{\delta} T:T}(s) - \hat{\mathcal{F}}_{\delta T:T}(s)\right|\nonumber\\[8pt]
			&\hskip0.3in = q(\tilde{\delta})\sup_{s\in[0,1]}\left|\frac{\delta-\tilde{\delta}}{1-\tilde{\delta}}\hat{\mathcal{F}}_{\tilde{\delta}T:\delta T}(s) + \frac{1-\delta}{1-\tilde{\delta}}\hat{\mathcal{F}}_{\delta T:T}(s) - \hat{\mathcal{F}}_{\delta T:T}(s)\right|\\[8pt]
			&\hskip0.3in= q(\tilde{\delta})\sup_{s\in[0,1]}\left|\frac{\delta-\tilde{\delta}}{1-\tilde{\delta}}\hat{\mathcal{F}}_{\tilde{\delta}T:\delta T}(s) - \frac{\delta-\tilde{\delta}}{1-\tilde{\delta}}\hat{\mathcal{F}}_{\delta T:T}(s) \right|\\[8pt]
			&\hskip0.3in= q(\tilde{\delta})\frac{\delta-\tilde{\delta}}{1-\tilde{\delta}}\sup_{s\in[0,1]}\left|\hat{\mathcal{F}}_{\tilde{\delta}T:\delta T}(s) -  \frac{\delta_0-\delta}{1-\delta}\hat{\mathcal{F}}_{\delta T:\delta_0 T}(s) - \frac{1-\delta_0}{1-\delta}\hat{\mathcal{F}}_{\delta_0 T:T}(s)\right|,
		\end{align}
		and following the same outline as above,
		\begin{align}
			&q(\tilde{\delta})\sup_{s\in[0,1]}\left|\hat{\mathcal{F}}_{\tilde{\delta} T:T}(s) - \hat{\mathcal{F}}_{\delta T:T}(s)\right|\nonumber\\[8pt]
			&\hskip0.3in \leq q(\tilde{\delta})\sup_{s\in[0,1]}\left| \hat{\mathcal{F}}_{\tilde{\delta}T:\delta T}(s)-\mathcal{F}_1(s)\right| + q(\tilde{\delta}) \sup_{s\in[0,1]}\left| \hat{\mathcal{F}}_{\delta T:\delta_0 T}(s)-\mathcal{F}_1(s)\right|\nonumber\\[8pt] 
			&\hskip0.6in + q(\tilde{\delta}) \sup_{s\in[0,1]}\left| \hat{\mathcal{F}}_{\delta_0 T:T}(s)-\mathcal{F}_T(s)\right| + q(\tilde{\delta})\frac{\delta-\tilde{\delta}}{1-\tilde{\delta}}\theta. \label{term2case1}
		\end{align}
		For the third term,
		\begin{align}
			&\left|q(\tilde{\delta}) - q(\delta)\right|\sup_{s\in[0,1]}\left|\hat{\mathcal{F}}_{1:\delta T}(s) - \hat{\mathcal{F}}_{\delta T:T}(s)\right|\nonumber\\[8pt]
			&\hskip0.3in \leq \left|q(\tilde{\delta}) - q(\delta)\right| \sup_{s\in[0,1]}\left|\hat{\mathcal{F}}_{1:\tilde{\delta} T}(s) - \mathcal{F}_1(s)\right| + \left|q(\tilde{\delta}) - q(\delta)\right| \sup_{s\in[0,1]}\left|\hat{\mathcal{F}}_{\tilde{\delta}T:\delta T} T(s) - \mathcal{F}_1(s)\right|\nonumber\\[8pt]
			&\hskip0.6in + \left|q(\tilde{\delta}) - q(\delta)\right| \sup_{s\in[0,1]}\left|\hat{\mathcal{F}}_{\delta T:\delta_0 T} T(s) - \mathcal{F}_1(s)\right|\nonumber\\[8pt]
			&\hskip0.6in + \left|q(\tilde{\delta}) - q(\delta)\right| \sup_{s\in[0,1]}\left|\hat{\mathcal{F}}_{\delta_0 T: T} T(s) - \mathcal{F}_T(s)\right| + \left|q(\tilde{\delta}) - q(\delta)\right|\theta\label{term3case1}.
		\end{align}
		Combining the terms from Equations \ref{term1case1}, \ref{term2case1}, and \ref{term3case1}, we obtain
		\begin{align}
			\left|\hat{Q}_T(\tilde{\delta}) - \hat{Q}_T(\delta)\right| &\leq \left(q(\tilde{\delta}) + \left|q(\tilde{\delta})-q(\delta)\right|\right)\sup_{s\in[0,1]}\left|\hat{\mathcal{F}}_{1:\tilde{\delta}T}(s) - \mathcal{F}_1(s)\right|\nonumber\\[6pt]
			&\hskip0.2in +\left(2q(\tilde{\delta}) + \left|q(\tilde{\delta})-q(\delta)\right|\right)\sup_{s\in[0,1]}\left|\hat{\mathcal{F}}_{\tilde{\delta}T:\delta T}(s) - \mathcal{F}_1(s)\right| \nonumber\\[6pt]
			&\hskip0.2in + \left(q(\tilde{\delta})+ \left|q(\tilde{\delta})-q(\delta)\right|\right)\sup_{s\in[0,1]} \left|\hat{\mathcal{F}}_{\delta T:\delta_0 T}(s)-\mathcal{F}_1(s)\right|\nonumber\\[6pt]
			&\hskip0.2in + \left(q(\tilde{\delta})+ \left|q(\tilde{\delta})-q(\delta)\right|\right)\sup_{s\in[0,1]} \left|\hat{\mathcal{F}}_{\delta_0 T: T}(s)-\mathcal{F}_T(s)\right| \nonumber\\[6pt]
			&\hskip0.2in+\left(q(\tilde{\delta})\frac{\delta-\tilde{\delta}}{1-\tilde{\delta}} + \left|q(\tilde{\delta})-q(\delta)\right|\right)\theta. \hskip0.2in\text{(Case 1: $\tilde{\delta}<\delta\leq\delta_0$)} 
			\label{combine}
		\end{align}
		Each $q$ term in the first four pieces of Equation \ref{combine} is bounded above: $q(\tilde{\delta})\leq 1/2$, $q(\delta)\leq 1/2$, and their difference $|q(\tilde{\delta})-q(\delta)|< 1/2$. We can substitute the coefficients $1$, $3/2$, $1$, and $1$ that appear below in Equation \ref{formthm7} to resemble $\hat{A}^1_T$ in Equation \ref{hatAT1}. All that remains is to manipulate the final term of Equation \ref{combine} to look like the form in Lemma \ref{newseq}.
		\vskip0in
		In the initial scenario $\tilde{\delta} < \delta \leq \delta_0$ where $\tilde{\delta}, \delta, \delta_0\in\left[\frac{1}{2} - \frac{1}{2}\sqrt{1-4\kappa^2}, \frac{1}{2} + \frac{1}{2}\sqrt{1-4\kappa^2}\right]$, we examine the final coefficient term of Equation \ref{combine}. For any  $q(\tilde{\delta})$ and $q(\delta)$, we can write $|q(\tilde{\delta})-q(\delta)| \leq \sqrt{|\tilde{\delta}(1-\tilde{\delta}) - \delta(1-\delta)|} \leq \sqrt{|\tilde{\delta}-\delta|}$, and the expression becomes
		\begin{align}
			q(\tilde{\delta})\frac{(\delta-\tilde{\delta})}{(1-\tilde{\delta})} + \left|q(\tilde{\delta})-q(\delta)\right| &\leq q(\tilde{\delta})\frac{(\delta-\tilde{\delta})}{(1-\tilde{\delta})} + (\delta-\tilde{\delta})^{1/2}.
			\label{finalpiece3}
		\end{align}
		For some constant $0\leq\alpha\leq1$, we note that $(\delta-\tilde{\delta}) \leq (\delta-\tilde{\delta})^\alpha$, and similarly for $\kappa < 1/2$, $(\delta-\tilde{\delta})^{1/2} < (\delta-\tilde{\delta})^\kappa$.
		\begin{align}
			q(\tilde{\delta})\frac{(\delta-\tilde{\delta})}{(1-\tilde{\delta})} + (\delta-\tilde{\delta})^{1/2} &\leq \frac{q(\tilde{\delta})}{(1-\tilde{\delta})}(\delta-\tilde{\delta})^{1/2} + (\delta-\tilde{\delta})^{1/2}\\
			&\leq \left(\frac{\tilde{\delta}^{1/2}}{(1-\tilde{\delta})^{1/2}} + 1\right)(\delta-\tilde{\delta})^{1/2}\\
			&< \left(\frac{\tilde{\delta}^{1/2}}{(1-\tilde{\delta})^{1/2}}+ 1\right)(\delta-\tilde{\delta})^{\kappa}
			\label{finalpiece4}
		\end{align}
		Equation \ref{finalpiece4} takes the form
		\begin{align}
			\left(\frac{\tilde{\delta}^{1/2}}{(1-\tilde{\delta})^{1/2}}+ 1\right)(\delta-\tilde{\delta})^{\kappa} & \leq C_1\left|\tilde{\delta}-\delta\right|^\kappa,
		\end{align}
		where the maximum value of $C_1$ will occur at large $\tilde{\delta}$. We use a slightly cleaner form of the domain restriction to write $\kappa^2\leq \frac{1}{2} - \frac{1}{2}\sqrt{1-4\kappa^2}\leq \tilde{\delta}, \delta \leq \frac{1}{2} + \frac{1}{2}\sqrt{1-4\kappa^2} \leq 1-\kappa^2$, and set $C_1 = (1/\kappa)\sqrt{1-\kappa^2} + 1$. With $C = (2/\kappa)\sqrt{1-\kappa^2} + 1<\infty$ defined above and $C_1\leq C$, we adjust the final term to look like that in Lemma \ref{newseq}.
		\begin{align}
			\left|\hat{Q}_T(\tilde{\delta}) - \hat{Q}_T(\delta)\right| &\leq \sup_{s\in[0,1]}\left|\hat{\mathcal{F}}_{1:\tilde{\delta}T}(s) - \mathcal{F}_1(s)\right|\nonumber\\[6pt]
			&\hskip0.2in +\frac{3}{2}\sup_{s\in[0,1]}\left|\hat{\mathcal{F}}_{\tilde{\delta}T:\delta T}(s) - \mathcal{F}_1(s)\right| \nonumber\\[6pt]
			&\hskip0.2in + \sup_{s\in[0,1]} \left|\hat{\mathcal{F}}_{\delta T:\delta_0 T}(s)-\mathcal{F}_1(s)\right|\nonumber\\[6pt]
			&\hskip0.2in + \sup_{s\in[0,1]} \left|\hat{\mathcal{F}}_{\delta_0 T: T}(s)-\mathcal{F}_T(s)\right| \nonumber\\[6pt]
			&\hskip0.2in+C\theta\left|\tilde{\delta}-\delta\right|^{\kappa}   \label{formthm7}\\[8pt]
			&\leq \hat{A}^1_T + \hat{B}_T\left\|\tilde{\delta}-\delta\right\|^{\kappa},
			\label{final}
		\end{align}
		where $\kappa>0$ is the small constant from $q(\delta)$ that appears in the domain restriction. The last term in Equation \ref{combine} is bounded above by the equivalent in Equation \ref{formthm7} for all $\tilde{\delta}<\delta\leq \delta_0$ where $\tilde{\delta},\delta,\delta_0\in\Delta$ and when the trivial boundary condition $\kappa<1/2$ is satisfied.
		\vskip0in
		In case two of six where $\delta<\tilde{\delta}\leq \delta_0$, we can write
		\begin{align}
			\left|\hat{Q}_T(\tilde{\delta}) - \hat{Q}_T(\delta)\right| &\leq \left(q(\tilde{\delta}) + \left|q(\tilde{\delta})-q(\delta)\right|\right)\sup_{s\in[0,1]}\left|\hat{\mathcal{F}}_{1:\delta T}(s) - \mathcal{F}_1(s)\right|\nonumber\\[6pt]
			&\hskip0.2in +\left(2q(\tilde{\delta}) + \left|q(\tilde{\delta})-q(\delta)\right|\right)\sup_{s\in[0,1]}\left|\hat{\mathcal{F}}_{\delta T:\tilde{\delta} T}(s) - \mathcal{F}_1(s)\right| \nonumber\\[6pt]
			&\hskip0.2in + \left(q(\tilde{\delta})+ \left|q(\tilde{\delta})-q(\delta)\right|\right)\sup_{s\in[0,1]} \left|\hat{\mathcal{F}}_{\tilde{\delta} T:\delta_0 T}(s)-\mathcal{F}_1(s)\right|\nonumber\\[6pt]
			&\hskip0.2in + \left(q(\tilde{\delta})+ \left|q(\tilde{\delta})-q(\delta)\right|\right)\sup_{s\in[0,1]} \left|\hat{\mathcal{F}}_{\delta_0 T: T}(s)-\mathcal{F}_T(s)\right| \nonumber\\[6pt]
			&\hskip0.2in+\left(q(\tilde{\delta})\frac{\delta-\tilde{\delta}}{1-\tilde{\delta}} + \left|q(\tilde{\delta})-q(\delta)\right|\right)\theta. \hskip0.2in\text{(Case 2: $\delta<\tilde{\delta}\leq\delta_0$)} 
			\label{combine2}
		\end{align}
		We obtain the coefficients for $\hat{A}^2_T$ in Equation \ref{hatAT2} of $1$, $3/2$, $1$, and $1$ from the bounds $q(\tilde{\delta})\leq 1/2$, $q(\delta)\leq 1/2$, and $|q(\tilde{\delta})-q(\delta)|< 1/2$. In the same process as above, we can show the final piece is bounded by $C\theta|\tilde{\delta}-\delta|^{\kappa}$ where $C=(2/\kappa)\sqrt{1-\kappa^2} + 1$.
		\vskip0in
		For case three,
		\begin{align}
			\left|\hat{Q}_T(\tilde{\delta}) - \hat{Q}_T(\delta)\right| &\leq \left(q(\tilde{\delta}) + \left|q(\tilde{\delta})-q(\delta)\right|\right)\sup_{s\in[0,1]}\left|\hat{\mathcal{F}}_{1:\tilde{\delta} T}(s) - \mathcal{F}_1(s)\right|\nonumber\\[6pt]
			&\hskip0.2in +\left(2q(\tilde{\delta}) + \left|q(\tilde{\delta})-q(\delta)\right|\right)\sup_{s\in[0,1]}\left|\hat{\mathcal{F}}_{\tilde{\delta} T:\delta_0 T}(s) - \mathcal{F}_1(s)\right| \nonumber\\[6pt]
			&\hskip0.2in + \left(2q(\tilde{\delta})+ \left|q(\tilde{\delta})-q(\delta)\right|\right)\sup_{s\in[0,1]} \left|\hat{\mathcal{F}}_{\delta_0 T:\delta T}(s)-\mathcal{F}_T(s)\right|\nonumber\\[6pt]
			&\hskip0.2in + \left(q(\tilde{\delta})+ \left|q(\tilde{\delta})-q(\delta)\right|\right)\sup_{s\in[0,1]} \left|\hat{\mathcal{F}}_{\delta T: T}(s)-\mathcal{F}_T(s)\right| \nonumber\\[6pt]
			&\hskip0.2in+\left(q(\tilde{\delta})\frac{\delta-\tilde{\delta}}{\delta}+q(\tilde{\delta})\frac{\delta-\tilde{\delta}}{1-\tilde{\delta}} + \left|q(\tilde{\delta})-q(\delta)\right|\right)\theta, \hskip0.2in\text{(Case 3: $\tilde{\delta}\leq \delta_0 < \delta$)} 
			\label{combine3}
		\end{align}
		we obtain the coefficients for $\hat{A}^3_T$ in Equation \ref{hatAT3} of $1$, $3/2$, $3/2$, and $1$, and $C=(2/\kappa)\sqrt{1-\kappa^2}+1$ as above when $\tilde{\delta}$ is large and $\delta$ small.
		\vskip0in
		For case four,
		\begin{align}
			\left|\hat{Q}_T(\tilde{\delta}) - \hat{Q}_T(\delta)\right| &\leq \left(q(\tilde{\delta}) + \left|q(\tilde{\delta})-q(\delta)\right|\right)\sup_{s\in[0,1]}\left|\hat{\mathcal{F}}_{1:\delta T}(s) - \mathcal{F}_1(s)\right|\nonumber\\[6pt]
			&\hskip0.2in +\left(2q(\tilde{\delta}) + \left|q(\tilde{\delta})-q(\delta)\right|\right)\sup_{s\in[0,1]}\left|\hat{\mathcal{F}}_{\delta T:\delta_0 T}(s) - \mathcal{F}_1(s)\right| \nonumber\\[6pt]
			&\hskip0.2in + \left(2q(\tilde{\delta})+ \left|q(\tilde{\delta})-q(\delta)\right|\right)\sup_{s\in[0,1]} \left|\hat{\mathcal{F}}_{\delta_0 T:\tilde{\delta} T}(s)-\mathcal{F}_T(s)\right|\nonumber\\[6pt]
			&\hskip0.2in + \left(q(\tilde{\delta})+ \left|q(\tilde{\delta})-q(\delta)\right|\right)\sup_{s\in[0,1]} \left|\hat{\mathcal{F}}_{\tilde{\delta} T: T}(s)-\mathcal{F}_T(s)\right| \nonumber\\[6pt]
			&\hskip0.2in+\left(q(\tilde{\delta})\frac{\delta-\tilde{\delta}}{\tilde{\delta}}+q(\tilde{\delta})\frac{\delta-\tilde{\delta}}{1-\delta} + \left|q(\tilde{\delta})-q(\delta)\right|\right)\theta, \hskip0.2in\text{(Case 4: $\delta\leq \delta_0 < \tilde{\delta}$)}
			\label{combine4}
		\end{align}
		we obtain the coefficients for $\hat{A}^4_T$ in Equation \ref{hatAT4} of $1$, $3/2$, $3/2$, and $1$, and $C=(2/\kappa)\sqrt{1-\kappa^2}+1$ as above when $\tilde{\delta}$ is small and $\delta$ large.
		\vskip0in
		For case five, 
		\begin{align}
			\left|\hat{Q}_T(\tilde{\delta}) - \hat{Q}_T(\delta)\right| &\leq \left(q(\tilde{\delta}) + \left|q(\tilde{\delta})-q(\delta)\right|\right)\sup_{s\in[0,1]}\left|\hat{\mathcal{F}}_{1:\delta_0 T}(s) - \mathcal{F}_1(s)\right|\nonumber\\[6pt]
			&\hskip0.2in +\left(q(\tilde{\delta}) + \left|q(\tilde{\delta})-q(\delta)\right|\right)\sup_{s\in[0,1]}\left|\hat{\mathcal{F}}_{\delta_0 T:\tilde{\delta} T}(s) - \mathcal{F}_T(s)\right| \nonumber\\[6pt]
			&\hskip0.2in + \left(2q(\tilde{\delta})+ \left|q(\tilde{\delta})-q(\delta)\right|\right)\sup_{s\in[0,1]} \left|\hat{\mathcal{F}}_{\tilde{\delta} T:\delta T}(s)-\mathcal{F}_T(s)\right|\nonumber\\[6pt]
			&\hskip0.2in + \left(q(\tilde{\delta})+ \left|q(\tilde{\delta})-q(\delta)\right|\right)\sup_{s\in[0,1]} \left|\hat{\mathcal{F}}_{\delta T: T}(s)-\mathcal{F}_T(s)\right| \nonumber\\[6pt]
			&\hskip0.2in+\left(q(\tilde{\delta})\frac{\delta-\tilde{\delta}}{\tilde{\delta}}+ \left|q(\tilde{\delta})-q(\delta)\right|\right)\theta, \hskip0.2in\text{(Case 5: $\delta_0<\tilde{\delta}<\delta$)}
			\label{combine5}
		\end{align}
		we obtain the coefficients for $\hat{A}^5_T$ in Equation \ref{hatAT5} of $1$, $1$, $3/2$, and $1$, and $C=(2/\kappa)\sqrt{1-\kappa^2}+1$ as above when $\tilde{\delta}$ is small.
		\vskip0in
		For the sixth and final case, 
		\begin{align}
			\left|\hat{Q}_T(\tilde{\delta}) - \hat{Q}_T(\delta)\right| &\leq \left(q(\tilde{\delta}) + \left|q(\tilde{\delta})-q(\delta)\right|\right)\sup_{s\in[0,1]}\left|\hat{\mathcal{F}}_{1:\delta_0 T}(s) - \mathcal{F}_1(s)\right|\nonumber\\[6pt]
			&\hskip0.2in +\left(q(\tilde{\delta}) + \left|q(\tilde{\delta})-q(\delta)\right|\right)\sup_{s\in[0,1]}\left|\hat{\mathcal{F}}_{\delta_0 T:\delta T}(s) - \mathcal{F}_T(s)\right| \nonumber\\[6pt]
			&\hskip0.2in + \left(2q(\tilde{\delta})+ \left|q(\tilde{\delta})-q(\delta)\right|\right)\sup_{s\in[0,1]} \left|\hat{\mathcal{F}}_{\delta T:\tilde{\delta} T}(s)-\mathcal{F}_T(s)\right|\nonumber\\[6pt]
			&\hskip0.2in + \left(q(\tilde{\delta})+ \left|q(\tilde{\delta})-q(\delta)\right|\right)\sup_{s\in[0,1]} \left|\hat{\mathcal{F}}_{\tilde{\delta} T: T}(s)-\mathcal{F}_T(s)\right| \nonumber\\[6pt]
			&\hskip0.2in+\left(q(\tilde{\delta})\frac{\delta-\tilde{\delta}}{\tilde{\delta}}+ \left|q(\tilde{\delta})-q(\delta)\right|\right)\theta, \hskip0.2in\text{(Case 6: $\delta_0<\delta<\tilde{\delta}$)}
			\label{combine6}
		\end{align}
		we obtain the coefficients for $\hat{A}^6_T$ in Equation \ref{hatAT6} of $1$, $1$, $3/2$, and $1$, and $C=(2/\kappa)\sqrt{1-\kappa^2}+1$ as above when $\tilde{\delta}$ is small.
		\vskip0in
		For all cases $\tilde{\delta},\delta,\delta_0\in\Delta$, the result of Lemma \ref{newseq} holds with the corresponding $\hat{A}^i_T$, $\hat{B}_T=\theta\left[(2/\kappa)\sqrt{1-\kappa^2}+1\right]$, and $\alpha = \kappa$. Thus, $\hat{Q}_T(\delta)$ is stochastically equicontinuous and uniformly converges in probability to $Q_0(\delta)$ on the interval $\Delta$.
	\end{proof}
	\par
	We now employ Theorem 2.1 from \citet{newey94} and Lemma \ref{c4} to complete the proof of Theorem \ref{consistent}.
	\begin{proof}[Proof of Theorem \ref{consistent}]
		We begin with condition (1). It can easily be seen that
		the quantity in Equation \ref{result1} is uniquely maximized to a value of $\theta$ at $\delta = \delta_0$, thus  
		$Q_0(\delta)$ in Equation \ref{q0pieces} has a unique maximum at $\delta = \delta_0$. If we relax the boundary restriction  $\delta\in\Delta$ and let $\delta\in(0,1)$, the unique maximum will still hold provided $\kappa \leq \delta_0^{1/2}(1-\delta_0)^{1/2}$, or $\delta_0\in\Delta$.  
		\vskip0in
		Condition (2) is satisfied as $\Delta$ is a closed and bounded set on $\mathbb{R}$. 
		\vskip0in
		For condition (3), the individual pieces of Equation \ref{q0pieces} are continuous, and we examine the extreme points of each subdomain. Because the set $\Delta$ is closed, continuity is trivially satisfied at the outward extremes. 
		For the point $\delta_0$,
		\begin{align}
			\lim_{\delta\rightarrow \delta_0^+} Q_0(\delta_0) &= Q_0(\delta_0)=\theta \delta_0^{1/2}(1-\delta_0)^{1/2}\label{llim}\\[8pt]
			\text{and}\hskip0.1in\lim_{\delta\rightarrow \delta_0^-} Q_0(\delta_0) &= \lim_{\delta\rightarrow \delta_0^-} \theta \frac{\delta_0(1-\delta)^{1/2}}{\delta^{1/2}} = \theta \delta_0^{1/2}(1-\delta_0)^{1/2}.\label{rlim}
		\end{align}
		\vskip0in
		Condition (4) is satisfied via Lemma \ref{c4}.
		\vskip0in
		The four conditions of Theorem 2.1 from \citet{newey94} are satisfied and $\hat{\delta}\xrightarrow{P}\delta_0$ implies $\hat{\tau}\xrightarrow{P}\tau$ for $t,\tau \in \left[\frac{T}{2}-\frac{T}{2}\sqrt{1-4\kappa^2}, \frac{T}{2}+\frac{T}{2}\sqrt{1-4\kappa^2}\right]$.
	\end{proof}
	\begin{ntheorem1}
		Let $\{S_t, \;t\in\mathbb{Z}\}$ be a stationary sequence such that $\mathcal{F}(s) = P(S_1\leq s)$ is Lipschitz continuous of order $C > 0$. Assume that $\{S_t, \;t\in\mathbb{Z}\}$ is $\mathcal{S}$-mixing and condition (1) of Definition \ref{def:smix} holds with $\gamma_m=m^{-AC}$, $\delta_m = m^{-A}$ for some $A>4$. Then the series
		\begin{align}
			\Gamma(s,s') = \sum_{-\infty < t < \infty} \mathbb{E}\left[S_1(s) S_t(s')\right]
		\end{align}
		converges absolutely for every choice of parameters $(s,s')\in\mathbb{R}^2$. Moreover, there exists a two-parameter Gaussian process $\mathcal{K}(s,t)$ such that $\mathbb{E}\left[\mathcal{K}(s,t)\right] = 0$,  $\mathbb{E}\left[\mathcal{K}(s,t)\;\mathcal{K}(s',t')\right] = (t \wedge t')\; \Gamma(s,s')$, and for some $\alpha>0$,
		\begin{align}
			\sup_{1 \leq t \leq T} \sup_{s\in[0,1]} \left| \sum_{i=1}^t \left(\mathbf{1}\{S_i \leq s\} - \mathcal{F}(s)\right) - \mathcal{K}(s,t)\right| = o\left(T^{1/2}(\log T)^{-\alpha}\right) \hskip0.1in \text{a.s.}
		\end{align}
	\end{ntheorem1}
	\begin{definition}
		A sequence of functions $\hat{Q}_T(\delta)$ is stochastically equicontinuous if for every $\varepsilon, \eta>0$ there exists a random quantity $\Gamma_T(\varepsilon,\eta)$ and a constant $T_0(\varepsilon,\eta)$ such that for $T\geq T_0(\varepsilon,\eta)$, $\mathbb{P}\left(\left|\Gamma_T(\varepsilon,\eta)\right|>\varepsilon\right)<\eta$ and for each $\delta$ there is an open set $\mathcal{N}(\delta,\varepsilon,\eta)$ containing $\delta$ with 
		\begin{align}
			\sup_{\tilde{\delta}\in\mathcal{N}(\delta,\varepsilon,\eta)}\left|\hat{Q}_T(\tilde{\delta})-\hat{Q}_T(\delta)\right| \leq \Gamma_T(\varepsilon,\eta),\hskip0.1in\text{for $T>T_0(\varepsilon,\eta)$.}
		\end{align}
		\label{sedef}
	\end{definition}
	\begin{ntheorem2}
		If there is a function $Q_0(\delta)$ such that 
		\begin{enumerate}
			\item[(1)] $Q_0(\delta)$ is uniquely maximized at $\delta_0$, $\delta_0 = \arg \max_{\delta\in\Delta} Q_0(\delta)$;
			\item[(2)] $\Delta$ is compact;
			\item[(3)] $Q_0(\delta)$ is continuous;
			\item[(4)] $\hat{Q}_T(\delta)$ converges uniformly in probability to $Q_0(\delta)$, $\sup_{\delta\in\Delta} \left|\hat{Q}_T(\delta)-Q_0(\delta)\right|\xrightarrow{P}0$;
		\end{enumerate}
		then $\hat{\delta} = \arg \max_{\delta\in\Delta} \hat{Q}_T(\delta) \xrightarrow{P} \delta_0$.
		\label{cp}
	\end{ntheorem2}
	
\end{document}